\documentclass{article}

\PassOptionsToPackage{numbers}{natbib}   

 \usepackage[preprint]{neurips_2026}

\usepackage[utf8]{inputenc} 
\usepackage[T1]{fontenc}    
\usepackage{hyperref}       
\usepackage{url}            
\usepackage{booktabs}       
\usepackage{amsfonts}       
\usepackage{nicefrac}       
\usepackage{microtype}      
\usepackage{xcolor}         
\usepackage{amsthm}
\usepackage{amsmath}
\usepackage{tikz}
\usepackage{booktabs}  
\usepackage{amssymb}

\theoremstyle{plain}
\newtheorem{theorem}{Proposition}
\newtheorem{corollary}[theorem]{Corollary}

\theoremstyle{definition}
\newtheorem{definition}{Definition}

\theoremstyle{remark}

 \title{Beyond ECE: Calibrated Size Ratio, Risk Assessment, and Confidence-Weighted Metrics}

\author{ 
Fernando Martin-Maroto \\ Mathematics of Behavior and Intelligence Lab.  \\ Champalimaud Research. Avenida de Brasília, Lisbon  \\  \texttt{fmmaroto@gmail.com} \\ 
\And
Nabil Abderrahaman-Elena  \\ Algebraic AI \\ Calle Castelló, 117, Madrid \\  \texttt{nabil.abderrahaman@gmail.com} \\ 
\And
Gonzalo G. de Polavieja \\ Mathematics of Behavior and Intelligence Lab. \\ Champalimaud Research. Avenida de Brasília, Lisbon  \\  \texttt{gonzalo.depolavieja@gmail.com} \\ 
}

\begin{document}

\maketitle

\begin{abstract}
Confidence calibration has been dominated by the Expected Calibration Error (ECE), a linear metric that counts calibration offset equally regardless of the confidence level at which it occurs. We show that ECE can remain small even under arbitrarily large overconfidence risk, so we propose Calibrated Size Ratio (CSR) instead, an interpretable metric that equals 1 under perfect calibration, from which we derive the risk probability $P_{\mathrm{risk}}$ that quantifies the statistical evidence for overconfidence. We further argue that overconfidence risk assessment must be complemented by a measure of discriminative value: whether the assigned confidences actively distinguish correct from incorrect predictions. We show that confidence-weighted accuracy $\mathrm{cwA}$ is the natural such complement, and that confidence-weighting extends to all standard classification metrics. In particular, we prove that the confidence-weighted AUC (cwAUC) captures the information about calibration while the classical AUC cannot. We validate the proposed indicators on several synthetic confidence distributions under multiple controlled calibration profiles and find that CSR separates risky from non-risky assignments. We also test the metrics on fifteen real datasets, with and without post-hoc calibration, and find that standard methods can yield risky confidence profiles.
\end{abstract}

\section{Introduction}

Reliable confidence estimation is fundamental to trustworthy machine learning. As ML systems are deployed in high-stakes domains, detecting and penalizing incorrect confidence assignments, and particularly overconfidence, becomes a safety requirement. The Expected Calibration Error (ECE)~\cite{guo2017calibration} has become the standard for evaluating confidence profiles, yet it has attracted substantial criticism~\cite{chidambaram2024flawed} — and we argue the problem runs deeper than previously recognized.

A well-calibrated model predicting confidence $c$ for an event should be correct with probability approximately $c$. ECE operationalizes this by measuring the expected absolute gap between confidence and accuracy over $M$ bins:
$$\text{ECE} = \sum_{m=1}^{M} \frac{|B_m|}{N} \left| \text{acc}(B_m) - \text{conf}(B_m) \right|$$
ECE does not merely suffer from well-known binning artifacts — it fails to extract the calibration information that actually matters. ECE is a linear metric: it treats a gap $\delta$ identically regardless of the confidence level at which it occurs. Yet a miscalibration of $\delta = 0.05$ at $\text{conf} = 0.55$ and at $\text{conf} = 0.95$ are not remotely equivalent failures. Since bin populations are largest at low confidence, ECE is dominated by the regime where miscalibration is least dangerous, and systematically blind to the tail where it is most consequential. We argue that calibration contributions must be measured in their natural units: a gap at confidence $c$ should be assessed relative to $(1-c)$, reflecting the true magnitude of the implied failure. 

Recent work has focused either on improving ECE by addressing binning artifacts (ECCE-MAD and ECCE-R~\cite{arrietaibarra2022metrics}, LS-ECE~\cite{chidambaram2024flawed}, AD-ECE)~\cite{vaicenavicius2019evaluating} or on introducing altogether new calibration indicators. Some indicators are closer to ECE in essence, like T-Cal~\cite{lee2023tcal}, a debiased squared version of ECE, or ECI~\cite{famiglini2023rigorous}, which uses normalised $L_2$ distances. Others differ more substantially: TCE~\cite{matsubara2023tce} uses a Binomial test to measure the percentage of predictions that are provably miscalibrated; CalDist~\cite{qiao2024distance} is the $L_1$ distance between the confidence scores and their isotonically calibrated values; CWSA and CWSA+~\cite{shahnazari2025cwsa} are two evaluation metrics that reward correct predictions in proportion to their confidence while penalising risky overconfident mistakes. ECD~\cite{sumler2025entropic}  computes the average difference between the model's negative predictive entropy and the log-probability assigned to the observed outcome. This indicator incorporates the $(1-c)$ scaling factor that ECE neglects via the odds ratio; however, in ECD this factor appears inside a logarithm, yielding an entropy-like quantity rather than directly scaling $(1 - \mathbf{1}[\hat{y} = y])$, for which it is the natural denominator. 

By separating two different aspects of calibration, we show that the simplest indicators are the most adequate. This work makes the following contributions:

- Calibrated Size Ratio (CSR): We introduce an interpretable calibration metric that equals 1 for perfectly calibrated assignments and gives a risk probability $P_{\mathrm{risk}}$ of overconfidence that becomes more accurate as $N$ grows. Furthermore, this probability can be obtained from a normal distribution, without the need of Monte Carlo resampling procedures required by frameworks like ECD or T-Cal.

- Confidence-Weighted Metrics and cwAUC. We argue that confidence-weighted accuracy $\mathrm{cwA}$ is the natural complement to CSR, capturing the practical value of the confidence profile. We show that the confidence-weighting approach extends naturally to all standard classification metrics. We further show that AUC is invariant to any monotone rescaling of confidence scores and therefore conveys no information about calibration. We prove that cwAUC admits a probabilistic pairwise interpretation and that $\text{cwAUC} - \text{AUC}$ captures the discriminative value added by calibration.

- Theoretical Analysis and Experimental Validation:  We provide theoretical foundations for the proposed indicators and validate them on several synthetic distributions under multiple controlled calibration profiles and dataset sizes. We also test them on real datasets with and without post-hoc calibration. Experiments demonstrate that CSR achieves near-perfect sensitivity and specificity in detecting overconfidence.

\section{Methodology}

\begin{definition}[Multi-Class Classification with Confidence Assignment] \label{def:conf_assignment_setup}
Given a dataset $\mathcal{D} = \{(x_i, y_i)\}_{i=1}^N$ with labels $y_i \in \{1, \ldots, K\}$, $K \geq 2$, each sample is assigned a predicted class $\hat{y}_i$ and a confidence score $\mathrm{conf}_i \in [0,1]$, also written $\mathrm{conf}_i^{(\hat{y}_i)}$, independently of all other samples. Per-class scores $\mathrm{conf}_i^{(k)} \in [0,1]$ may additionally be provided for each $k$.
\end{definition}

The independence condition is satisfied by standard confidence assignment methods, including post-hoc calibration methods fitted on a separate held-out set. It would not be satisfied if calibration is performed directly on $\mathcal{D}$ or if a transductive method is used to assign confidences.

\subsection{Calibrated Size Ratio}

\begin{definition}[Calibrated Size Ratio]\label{def:cal_dif}
Let $\mathcal{D} = \{(x_i, y_i)\}_{i=1}^N$ and $\{\hat{y}_i, \text{conf}_i\}_{i=1}^N$ be as in Definition \ref{def:conf_assignment_setup}. Assume  $\text{conf}_i < 1$ for all $i$, then: \[
\text{CSR} = \mathbb{E}_{\mathcal{D}}\left[\frac{1 - \mathbf{1}[\hat{y}_i = y_i]}{1 - \text{conf}_i}\right] = \frac{1}{N}\sum_{i=1}^N \frac{1 - \mathbf{1}[\hat{y}_i = y_i]}{1 - \text{conf}_i} = \frac{1}{N}\sum_{i:\hat{y}_i \neq y_i} \frac{1}{1 - \text{conf}_i}.
\]
\end{definition}
\smallskip

We use various notions of calibration, but for now consider the simplest frequentist notion, one that does not depend on a binning scheme:
\[
    \mathbb{E}_{\mathcal{D}}\big[\mathbf{1}[\hat{y}_i = y_i] \mid \mathrm{conf}_i = c\big] = c, 
    \quad \forall\, c \in [0,1].
\]
This condition can be seen as the limiting case of ECE-based calibration where each confidence value defines its own bin; 
we refer to this as \emph{perfect pointwise calibration}. Since

\begin{equation}
\text{CSR} = \mathbb{E}_{\mathcal{D}}\left[\frac{1 - \mathbf{1}[\hat{y}_i = y_i]}{1 - \text{conf}_i}\right] = \sum_c \left[p_{\mathcal{D}}(c) \frac{1 -  \mathbb{E}_{\mathcal{D}}[  \mathbf{1}[\hat{y}_i = y_i]  \mid \mathrm{conf}_i = c ]}{1 - c}\right] 
 \label{eq:CSR_conditional}
\end{equation}
if  $\mathcal{D}$ is perfectly pointwise calibrated: 
\[
\text{CSR} = \sum_c \left[p_{\mathcal{D}}(c) \frac{1 - c}{1 - c}\right] = \sum_c p_{\mathcal{D}}(c) = 1
\]

Suppose there exists an incorrect, high-confidence prediction for which $N(1 - \mathrm{conf}_i) < 1$; this sample alone pushes $\mathrm{CSR}$ beyond 1. More generally, by Proposition~\ref{proposition:jensen_CSR}, if the average confidence assigned to incorrect predictions $\mathbb{E}_{\mathcal{D}}[\mathrm{conf}_i \mid \hat{y}_i \neq y_i]$ exceeds the empirical accuracy $\bar{p}$ then $\mathrm{CSR} > 1$.  Proposition~\ref{proposition:jensen_CSR} can be proven from Jensen's inequality and states: \[
    \mathrm{CSR} \geq \frac{1 - \bar{p}}{1 - \mathbb{E}_{\mathcal{D}}[\mathrm{conf}_i \mid \hat{y}_i \neq y_i]}.
\]
Therefore, any confidence profile with $\mathrm{CSR} = 1$ must satisfy $\mathbb{E}_{\mathcal{D}}[\mathrm{conf}_i \mid \hat{y}_i \neq y_i] \leq \bar{p}$. More importantly, confidence profiles that are overconfident on incorrect predictions and therefore risky yield $\mathrm{CSR} > 1$. Similarly, underconfident profiles on incorrect predictions give $\mathrm{CSR} < 1$.

Since CSR = 1 is attained by pointwise calibrated confidence profiles, we can interpret the Calibrated Size Ratio as follows. Write  $\text{CSR} = \frac{1}{N}\sum_{i=1}^N \frac{1 - \mathbf{1}[\hat{y}_i = y_i]}{1 - \text{conf}_i} $ as $ 
1 = \frac{1}{\text{CSR} \cdot N}\sum_{i=1}^N \frac{1 - \mathbf{1}[\hat{y}_i = y_i]}{1 - \text{conf}_i}$. It follows that CSR measures how many times larger a calibrated dataset would need to be to observe the incorrect predictions at their assigned confidence levels by chance. It is the ratio between the size of a calibrated dataset and the actual size $N$; a CSR of 10 indicates that the errors at the confidence values we observe would be expected from a dataset 10 times larger.  
\bigskip

In order to quantify risk, we adopt a Bayesian notion of calibration. To this end, we assume that confidence scores are assigned to each sample from a limit distribution. Consider the following definition:

\begin{definition}[Perfect Limitwise Calibration]\label{def:limit_calibration}
We say the confidences are perfectly limitwise calibrated if there exists a probability measure $P_{\text{limit}}$ such that the observations $\{(\hat{y}_i, y_i, \text{conf}_i)\}_{i=1}^N$ are i.i.d. draws from a limit distribution such that $P_{\text{limit}}(\hat{y}_i = y_i \mid \text{conf}_i = c) = c \quad \forall c \in [0,1]$
\end{definition}
\smallskip
Now calibration becomes: \begin{equation}
    \mathbb{E}_{limit}\big[\mathbf{1}[\hat{y}_i = y_i] \mid \mathrm{conf}_i = c\big] = c, 
    \quad \forall\, c \in [0,1].
    \label{eq:pointwise_calibration}
\end{equation}

Proposition \ref{theorem:PerfectCal_is_CSR_one_limitwise} in the appendix gives a formal proof, but, schematically, by approaching \[
\mathbb{E}_{\mathcal{D}}[  \mathbf{1}[\hat{y}_i = y_i]  \mid \mathrm{conf}_i = c ]   \approx \mathbb{E}_{limit}[  \mathbf{1}[\hat{y}_i = y_i]  \mid \mathrm{conf}_i = c ] 
\] and substituting into the expression \ref{eq:CSR_conditional} : \[
\text{CSR} \approx  \sum_c \left[p_{\mathcal{D}}(c) \frac{1 -  \mathbb{E}_{limit}[  \mathbf{1}[\hat{y}_i = y_i]  \mid \mathrm{conf}_i = c ]}{1 - c}\right] = \sum_c p_{\mathcal{D}}(c) = 1
\]
Therefore, for confidence profiles that are limitwise calibrated we expect $\mathrm{CSR} \approx 1$, but not exactly. We should then characterize its variability under perfect limitwise calibration via its standard deviation (see Proposition~ \ref{theorem:std_csr_under_perfect_limitwise} ): \[ 
\sigma_{\text{CSR}} = \sqrt{\frac{1}{N}\mathbb{E}_{limit} \left[\frac{\text{conf}_i}{1-\text{conf}_i}\right]} \approx   \sqrt{\frac{1}{N^2}\sum_{i=1}^N \frac{\text{conf}_i}{1-\text{conf}_i}}  
\]
where we also have approached $\mathbb{E}_{limit}[\cdot ] \approx \mathbb{E}_{\mathcal{D}}[\cdot  ]$. While $\mathrm{CSR}$ depends only on the confidence scores of incorrect predictions, its standard deviation depends on the odds $\frac{\mathrm{conf}_i}{1 - \mathrm{conf}_i}$ of all predictions, correct or incorrect. Hence, for $\sigma_{\text{CSR}} $ to be well-defined, we require that $\mathrm{conf}_i < 1$ for all $i$.

We can then use the upper tail of the normal distribution to assess the probability of observing data at least as extreme assuming the null hypothesis of limitwise calibration is true. For $\mathrm{CSR} > 1$, i.e. above the expected mean, the number of standard deviations  $z = (\text{CSR} - 1)/\sigma_{\text{CSR}}$, the $z$ score, provides a principled measure of the risk associated with the confidence profile. The associated risk probability is $P_{\mathrm{risk}} = \Phi(z)$ where $\Phi$ denotes the standard normal cumulative distribution function.

\subsection{Confidence Weighted Accuracy}

Consider the following straightforward metric:
\smallskip

\begin{definition}[Confidence Weighted Accuracy]\label{def:cwA}
The \emph{Confidence Weighted Accuracy} is the fraction of total confidence 
mass assigned to correct predictions:
\[
    \mathrm{cwA} = \frac{\mathbb{E}_{\mathcal{D}}[\mathrm{conf}_i \cdot \mathbf{1}[\hat{y}_i = y_i]]}{\mathbb{E}_{\mathcal{D}}[\mathrm{conf}_i]}= \frac{\sum_{i=1}^N \mathrm{conf}_i \cdot \mathbf{1}[\hat{y}_i = y_i]}{\sum_{i=1}^N \mathrm{conf}_i}.
\]
\end{definition}
\smallskip

A simple rearrangement (Proposition~\ref{prop:cwA_covariance}) shows that cwA is an affine function of the covariance between correctness and confidence: \[
    \mathrm{cwA} = \bar{p} + \frac{\mathrm{Cov}_{\mathcal{D}}(\mathrm{conf}_i,\, \mathbf{1}[\hat{y}_i = y_i])}{\bar{c}}
\]
where $\bar{c} = \frac{1}{N} \sum_{i=1}^N \text{conf}_i$ is the average confidence.

An oracle that assigns confidence 1 to correct predictions and 0 to incorrect ones attains $\mathrm{cwA} = 1$. At the other extreme, if confidences are uniform across all predictions, and therefore, uninformative, we obtain \[
\mathrm{cwA} = \bar{p} 
\]
Therefore we can see $\mathrm{cwA}$ as the simplest measure of the accuracy obtained by using the confidence profile to weight the correctness of the predictions. A measure that complements Acc and CSR.
\smallskip

Furthermore, this simple approach of weighting predictions by their confidence can be extended to all other standard metrics used to evaluate models:
\smallskip

\begin{definition}[Confidence-Weighted Confusion Matrix] \label{definition:cw_CM}
For each class $k \in \{1, \ldots, K\}$ define: 
\[
\text{cwTP}^{(k)} = \sum_{i=1}^N \text{conf}_i \cdot \mathbf{1}[\hat{y}_i = y_i] \cdot \mathbf{1}[y_i = k] \,\,\,\,\,\,\,\,\,\,\,\,  \text{cwFN}^{(k)} = \sum_{i=1}^N \text{conf}_i \cdot \mathbf{1}[\hat{y}_i \neq y_i] \cdot \mathbf{1}[y_i = k] 
\]\[
\text{cwFP}^{(k)} = \sum_{i=1}^N \text{conf}_i \cdot \mathbf{1}[\hat{y}_i \neq y_i] \cdot \mathbf{1}[\hat{y}_i = k] \,\,\,\,\,\,\,\,\,\,\,
\text{cwTN}^{(k)} = \sum_{i=1}^N \text{conf}_i \cdot \mathbf{1}[\hat{y}_i \neq k] \cdot \mathbf{1}[y_i \neq k] \]
And let: \[
\text{cwP}^{(k)} = \sum_{i=1}^N \text{conf}_i \cdot \mathbf{1}[y_i = k]   \,\,\,\,\,\,\,\,\,\,\,\, \text{cwN}^{(k)} = \sum_{i=1}^N \text{conf}_i \cdot \mathbf{1}[y_i \neq k]
\]
\end{definition}

According to Proposition~\ref{proposition:cw_basic_relations}, the confidence-weighted counterparts of $\text{TP}^{(k)}$, $\text{FP}^{(k)}$, $\text{FN}^{(k)}$, and $\text{TN}^{(k)}$ satisfy the same four structural relations as their classical counterparts:  \[
\text{cwP}^{(k)} = \text{cwTP}^{(k)} + \text{cwFN}^{(k)} \quad \quad \text{cwN}^{(k)} = \text{cwFP}^{(k)} + \text{cwTN}^{(k)}
\]   \[
cwP^{(k)} + cwN^{(k)} = N\bar{c}    \quad \quad   \text{cwTP}^{(k)} + \text{cwFP}^{(k)} + \text{cwFN}^{(k)} + \text{cwTN}^{(k)} = N\bar c \quad \forall k.
\] 

This means that any algebraic identity that holds purely as a consequence of these structural relations carries over directly to the confidence-weighted setting. In particular, all standard classification metrics that are defined as combinations of TP, FP, FN, and TN, such as precision, recall (sensitivity), specificity, F1-score, or the Matthews correlation coefficient, admit confidence-weighted versions.  

An example is cwA itself, the confidence-weighted counterpart of the overall accuracy, which satisfies relations analogous to those of accuracy. Indeed, Propostitions \ref{cwa_from_standard_metrics} and \ref{theorem:cwA_macro} show that:
\[
\text{cwA} = \frac{\sum_{k=1}^K \text{cwTP}^{(k)}}{\sum_{k=1}^K \text{cwP}^{(k)}} \quad \quad and \quad \quad \sum_{k=1}^K \text{cwAcc}^{(k)} = (K-2) + 2\,\text{cwA}
\]

Another example is the classic result~\cite{mohri2018auc} that the area under the ROC~\cite{green1966signal} curve equals the probability that a randomly chosen positive prediction ranks higher than a randomly chosen negative prediction: \[
 \text{AUC}^{(k)} = \mathbb{E}_{x^+ \sim \text{Unif}, x^- \sim \text{Unif}} \left[ \mathbf{1}[\text{conf}_i^{(k)} > \text{conf}_j^{(k)}] + \frac{1}{2}\mathbf{1}[\text{conf}_i^{(k)} = \text{conf}_j^{(k)}] \right]
\]
and it also holds for the confidence-weighted counterpart (see Theroem \ref{thorem:AUC}):
\[
 \text{cwAUC}^{(k)} = \mathbb{E}_{x^+ \sim \mathcal{D}^+_w, x^- \sim \mathcal{D}^-_w} \left[ \mathbf{1}[\text{conf}_i^{(k)} > \text{conf}_j^{(k)}] + \frac{1}{2}\mathbf{1}[\text{conf}_i^{(k)} = \text{conf}_j^{(k)}] \right]
\] 
Both results are analogous, but the sampling differs: in the classic AUC, samples are drawn uniformly, whereas in its confidence-weighted counterpart they are drawn with probability proportional to the confidence assigned to the predicted class.
 
Furthermore, Proposition \ref{thorem:AUC_diffs} establishes the follwing result: define for each pair $(i,j)$ with $i \in \mathcal{D}^+_k$, $j \in \mathcal{D}^-_k$  (i.e, $y_i = k$ and $y_j \not= k$). Let $z_{ij} =\mathbf{1}[\text{conf}_i^{(k)} > \text{conf}_j^{(k)}] + \frac{1}{2}\mathbf{1}[\text{conf}_i^{(k)} = \text{conf}_j^{(k)}] \quad \text{and}$ and $w_{ij} = \text{conf}_i \cdot \text{conf}_j$. Then: \[
\text{cwAUC}^{(k)} - \text{AUC}^{(k)} = \frac{\text{Cov}(w_{ij}, z_{ij})}{\mathbb{E}[w_{ij}]}
\]
Hence, the gap cwAUC $-$ AUC is positive when the model is more confident on pairs it ranks correctly than on pairs it ranks incorrectly.   The metric $\text{cwAUC}^{(k)}$ measures the correlation between the certainty the model assigns to its predicted classes and the ranking performance for class $k$: a pair of positive sample $i$ and negative sample $j$ contributes to cwAUC proportionally to $\text{conf}_i^{(\hat{y}_i)} \cdot \text{conf}_j^{(\hat{y}_j)}$, the product of the confidence the model assigns to each sample's predicted class, so that comparisons where both predictions are made with high certainty carry greater weight.   Notice that the confidences compared correspond to class $k$ while the weight $ \text{conf}_i^{(\hat{y}_i)} \cdot \text{conf}_j^{(\hat{y}_j)} $ involves the confidences assigned to the predicted classes $\hat{y}$ that may or may not be equal to $k$.

Let $\phi : \mathbb{R} \to \mathbb{R}$ be any strictly monotone increasing function. Then (see Proposition \ref{cor:auc_invariance}):
\[
\text{AUC}^{(k)}\!\left(\phi \circ \text{conf}^{(k)}\right) = \text{AUC}^{(k)}\!\left(\text{conf}^{(k)}\right)
\]
Since any strictly monotone post-hoc calibration method leaves the ranking of $\text{conf}^{(k)}$ unchanged, $\text{AUC}^{(k)}$ is invariant to calibration.  By contrast, $\text{cwAUC}^{(k)}$ is sensitive to the magnitudes of $\text{conf}^{(\hat{y})}$ through the pairwise weights $w_{ij}$. Therefore $\text{cwAUC}^{(k)} - \text{AUC}^{(k)}$ measures whatever discriminative value is gained from calibration. In practice, deviations in $\text{AUC}^{(k)}$ may arise under weakly monotone calibrators (e.g., isotonic) or per-sample cross-class renormalization.

\begin{figure}[htbp]
    \centering
    \begin{minipage}{0.49\textwidth}
        \centering
        \includegraphics[width=\textwidth]{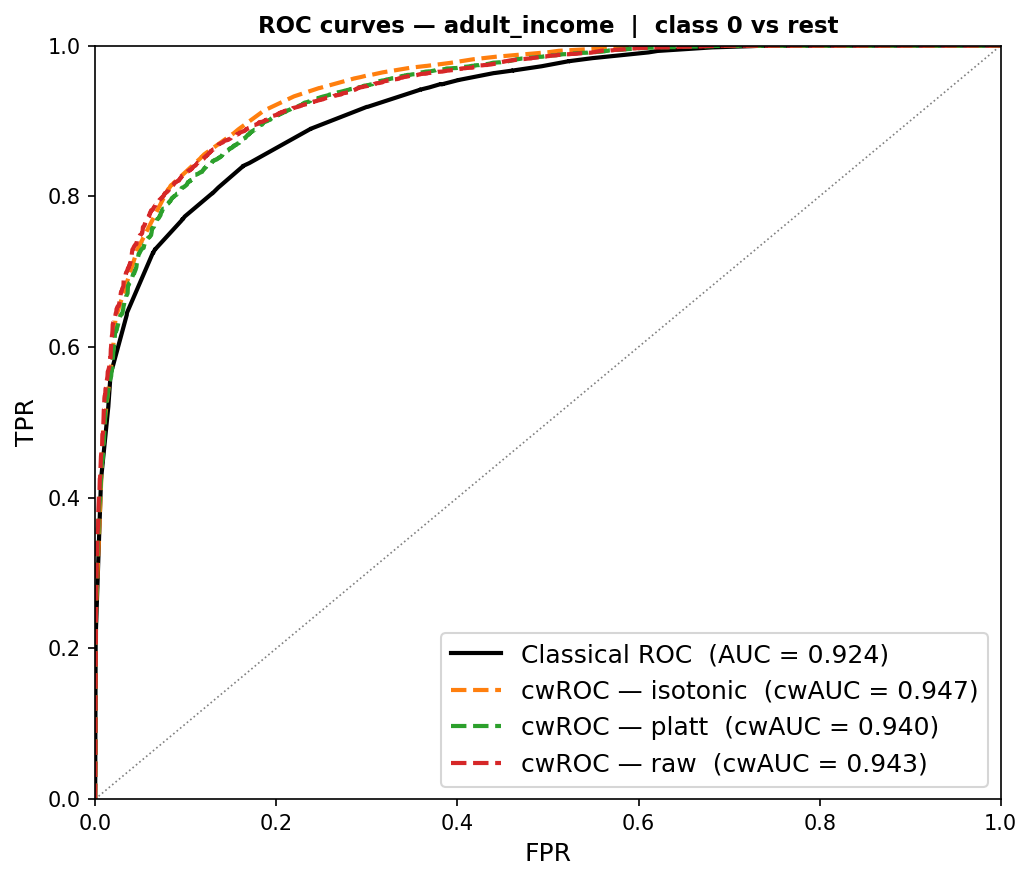}
    \end{minipage}
    \hfill
    \begin{minipage}{0.49\textwidth}
        \centering
        \includegraphics[width=\textwidth]{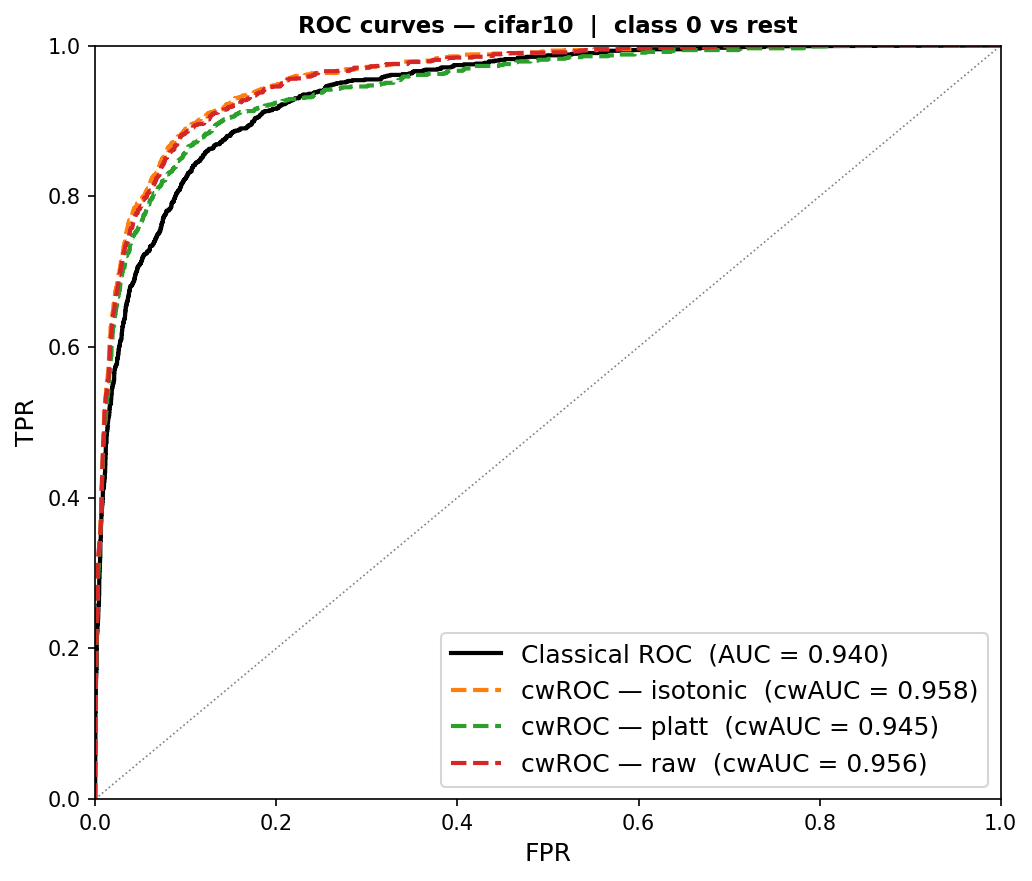}
    \end{minipage}
    \caption{ROC and confidence-weighted ROC curves for the binary Adult Income~\cite{uci_adult_income_1996} and multiclass CIFAR-10~\cite{krizhevsky2009learning} under raw XGBoost confidences, isotonic~\cite{zadrozny2002transforming} and Platt~\cite{platt1999probabilistic} post-hoc calibrations. The classical ROC curve (black solid) uses confidence scores only as a ranking signal and is therefore similar across calibration methods (see Proposition \ref{cor:auc_invariance} and its remark). The additional area of each cwROC curve (dashed) over the classical ROC measures the discriminative value added by calibration beyond ranking.}
   \label{fig:my_figure_AUC_adult_income}
\end{figure}

\subsection{Independence from ECE}

Confidence weighted accuracy and Calibrated Size Ratio measure independent aspects of a confidence profile. It is possible to have a safe profile with $\mathrm{CSR} \approx 1$ that is completely uninformative, and equally, a highly informative profile can be either risky or safe. We now establish, via explicit counterexamples, that both measures are also independent of the Expected Calibration Error ($\mathrm{ECE}$).

One or a few high-confidence incorrect predictions can yield a low ECE and an unbounded amount of risk. A formal statement is given in Proposition~\ref{theorem:csr_and_ece_ortogonality}: given a dataset $\mathcal{D}$ with at least one incorrect prediction, for any $\Lambda > 0$ there exists a confidence profile $\{\mathrm{conf}_i\}_{i=1}^N$ such that: \[
\mathrm{ECE} < \frac{1}{N} \quad \text{and} \quad \mathrm{CSR} > \Lambda. 
\] Therefore, a low ECE provides no risk guarantee under any fixed binning scheme.

The oracle confidence profile, which assigns confidence 1 to correct predictions and 0 to incorrect ones, is maximally informative, yields $\mathrm{cwA} = 1$ and $\mathrm{ECE} = 0$. On the other hand, a flat confidence profile in which $\mathrm{conf}_i = \bar{p}$ for every $i$ also yields $\mathrm{ECE} = 0$ and is even pointwise calibrated, yet has $\mathrm{cwA} = \bar{p}$. Hence $\mathrm{ECE} = 0$, regardless of the number of bins, is independent of cwA and compatible with completely uninformative confidence profiles.

\section{Experiments}

\subsection{Synthetic Datasets}

To evaluate the behaviour of  $\mathrm{CSR}$,  $\mathrm{cwA}$ and other confidence-weighted metrics across a controlled range of confidence score distributions, we construct ten synthetic distributions, each drawn with a fixed random seed for reproducibility. The distributions, summarised in Appendix Table~\ref{tab:distributions}, are designed to cover qualitatively distinct regimes of model confidence, including uniform, skewed, bimodal, narrow-support, log-space, and bell-shaped profiles. Experiments are conducted at sample sizes $N \in \{100, 1{,}000, 10{,}000, 1{,}000{,}000\}$.

For each distribution, ground-truth labels are generated under eight calibration regimes that define a mapping $p_{\text{true}}(c)$ from confidence value $c$ to true label probability, also summarised inTable~\ref{tab:distributions}. For each sample $i$, a predicted class $\hat{y}_i \in \{0, 1\}$ is drawn uniformly at random, reflecting a balanced class distribution. A correctness indicator is then drawn as $\text{correct}_i \sim \text{Bernoulli}(p_{\text{correct}}(c_i))$, where $c_i$ is the confidence score and $p_{\text{correct}}$ is the calibration mode function. If $\text{correct}_i = 1$, the ground truth label is set to $g_i = \hat{y}_i$. If $\text{correct}_i = 0$, the ground truth is set to the complementary class $g_i = 1 - \hat{y}_i$. Under this construction, $c_i$ represents the model's confidence in its predicted class $\hat{y}_i$, and the calibration mode $p_{\text{correct}}(c_i)$ controls how well that confidence reflects the true probability of being correct. A perfectly calibrated model corresponds to $p_{\text{correct}}(c) = c$, meaning that among all samples assigned confidence $c$, exactly a fraction $c$ are correctly predicted.

Table~\ref{tab:report_perfect_synthetic} reports the metric values obtained on perfectly calibrated synthetic datasets across varying sample sizes ($N \in \{100, 10{,}000\}$), averaged over 100 repetitions. Under perfect calibration, the true probability of correct prediction equals the confidence score, $p_{\text{correct}}(c) = c$, so any deviation of $\mathrm{CSR}$ from 1 is purely due to finite-sample noise.

The $>1\boldsymbol{\sigma}_{\mathrm{CSR}}$ and $>3\boldsymbol{\sigma}_{\mathrm{CSR}}$ columns report the percentage of repetitions in which $\mathrm{CSR}$ exceeds 1 by more than one and three standard deviations respectively. Under perfect calibration these are false positives. The normal distribution predicts false positive rates of 15.87\% and 0.13\% for the one- and three-sigma thresholds. The experimental results show that for the one-sigma threshold, 293 out of 3,000 experiments ($\approx 9.77\%$) exceed the threshold, below the theoretical 15.87\%, suggesting the test is conservative at this level. For the three-sigma threshold, at $N = 100$, 7 out of 1,000 experiments ($0.70\%$) exceed the threshold, above the theoretical 0.13\%, indicating that the Gaussian approximation underestimates the tails of $\mathrm{CSR}$ at small sample sizes; at $N = 1{,}000{,}000$, only 3 out of 1,000 ($0.30\%$) do so, now closer to the theoretical rate. This behaviour is consistent with the central limit theorem: the distribution converges to Gaussian as $N$ grows. Overall, the empirical false positive rates remain close to their theoretical predictions across all sample sizes and distributions.

Table~\ref{tab:report_perfect_synthetic} also reports the average value of $P_{\mathrm{risk}} = \Phi(z)$, which quantifies the probability that the observed $\mathrm{CSR}$ deviation is not due to chance. Across all distributions and sample sizes, the average $P_{\mathrm{risk}}$ remains below 50\%, confirming that under perfect calibration the distributions are not considered risky on average.

\begin{table}[ht]
\centering
\footnotesize
\begin{tabular}{lrrrrrrrr}
\toprule
\textbf{Distribution} & $\overline{\textbf{Acc}}$ & $\overline{\textbf{cwA}}$ & $\overline{\textbf{gain}}$ & $\overline{\textbf{CSR}}$ & $\overline{\boldsymbol{\sigma}_{\mathrm{CSR}}}$ & $>1\boldsymbol{\sigma}_{\mathrm{CSR}}$ & $>3\boldsymbol{\sigma}_{\mathrm{CSR}}$ & $\overline{\textbf{P}_{\textbf{risk}}}$ \\
\midrule
Uniform & 0.4999 & 0.6655 & 33.18\% & 1.0082 & 0.1201 & 9.33\% & 1.67\% & 29.10\% \\
Skew High & 0.8601 & 0.8911 & 22.14\% & 0.9897 & 12.6451 & 1.33\% & 0.33\% & 16.89\% \\
Skew Low & 0.1436 & 0.3340 & 22.28\% & 0.9997 & 0.0189 & 16.33\% & 0.00\% & 38.20\% \\
Bimodal & 0.4991 & 0.8091 & 61.94\% & 1.0671 & 5.7588 & 2.33\% & 0.67\% & 16.81\% \\
Tight Hi & 0.9012 & 0.9049 & 3.77\% & 0.9910 & 0.2448 & 11.00\% & 1.00\% & 33.23\% \\
Tight Lo & 0.0998 & 0.1331 & 3.71\% & 1.0000 & 0.0128 & 14.33\% & 0.00\% & 36.06\% \\
Normal & 0.6977 & 0.7118 & 4.68\% & 1.0066 & 0.0723 & 14.33\% & 1.00\% & 36.46\% \\
Log-Uniform Low & 0.1092 & 0.4976 & 43.69\% & 0.9976 & 0.0309 & 9.33\% & 0.33\% & 32.45\% \\
Log-Uniform High & 0.9348 & 0.9615 & 41.24\% & 0.7953 & 10.2331 & 2.67\% & 0.67\% & 8.43\% \\
Bell & 0.4984 & 0.5439 & 9.10\% & 1.0038 & 0.0422 & 16.67\% & 0.00\% & 35.13\% \\
\bottomrule
\end{tabular}
\caption{Behaviour of cwA and CSR for synthetic confidence distributions under perfect calibration, averaged over 100 repetitions for each of $N \in \{100, 10{,}000, 100{,}000\}$. Values of $>1\boldsymbol{\sigma}_{\mathrm{CSR}}$ and $>3\boldsymbol{\sigma}_{\mathrm{CSR}}$ report the percentage of repetitions in which CSR exceeded its mean by more than 1 and 3 standard deviations (tail risks of 84.13\% and 99.87\% respectively). Remaining indicators are averaged across repetitions and dataset sizes. Table~\ref{tab:report_perfect_synthetic} shows results disaggregated by $N$.}
\label{tab:report_perfect_synthetic_summary}
\end{table}

The average accuracy ($\mathrm{Acc}$) and confidence-weighted accuracy ($\mathrm{cwA}$) are also reported in Table~\ref{tab:report_perfect_synthetic}, across all distributions and sample sizes. Under perfect calibration, $\mathrm{cwA}$ exceeds $\mathrm{Acc}$, as weighting predictions by their confidence gives higher weight to samples the model is more likely to get right. We additionally report the average gain, defined as $
\mathrm{gain} = \frac{\mathrm{cwA} - \mathrm{Acc}}{1 - min(\mathrm{cwA},  \mathrm{Acc})}$ expressed as a percentage. This quantity measures what proportion of the gap between the plain accuracy and perfect accuracy ($=1$) is gained by confidence weighting. The magnitude of the gain varies across distributions: distributions with mass concentrated near the extremes (such as Skew High, Log-Uniform High and Bimodal) exhibit larger gains, as the model's high-confidence predictions are more reliably correct. 

\subsection{Real Datasets}

\begin{table}[ht]
\centering
\footnotesize

\begin{tabular}{lrrrrrrrr}
\toprule
calibration & $\overline{\textbf{Acc}}$ & $\overline{\textbf{cwA}}$ & $\overline{\textbf{gain}}$ & $\overline{\textbf{CSR}}$ & $\overline{\boldsymbol{\sigma}_{\mathrm{CSR}}}$ & $>1\boldsymbol{\sigma}_{\mathrm{CSR}}$ & $>3\boldsymbol{\sigma}_{\mathrm{CSR}}$ & $\overline{\textbf{P}_{\textbf{risk}}}$ \\
\midrule
no post-hoc & 0.8552 & 0.8757 & 17.6\% & 1.1231 & 1.1854 & 6/15 & 2/15 & 46.22\% \\
isotonic & 0.8552 & 0.8841 & 22.4\% & 326166.5578 & 176.2217 & 11/15 & 10/15 & 83.46\% \\
platt & 0.8552 & 0.8789 & 17.0\% & 0.9004 & 0.1618 & 3/15 & 0/15 & 21.96\% \\
\bottomrule
\end{tabular}
\caption{> Indicators averaged over 15 datasets under three confidence calibration regimes applied to XGBoost outputs. Isotonic regression achieves the highest $\overline{\text{cwA}}$ (0.8841) but at the cost of substantially higher risk: CSR exceeds $1\boldsymbol{\sigma}_{\mathrm{CSR}}$ in 11/15 datasets and $3\boldsymbol{\sigma}_{\mathrm{CSR}}$ in 10/15, compared to 6/15 and 2/15 for raw output. Platt scaling offers the safest profile (3/15 and 0/15) with competitive cwA. See Appendix Table~\ref{tab:report_real} for per-dataset results.}
\label{tab:report_real_datasets_averaged}
\end{table}

We evaluate on fifteen classification datasets (ten binary, five multiclass) from the UCI Machine Learning Repository~\cite{dua2019uci}, CIFAR-10 and MNIST~\cite{lecun1998gradient}, modelled with XGBoost~\cite{chen2016xgboost}. For each dataset, post-hoc calibration is fitted on a held-out validation set, yielding three confidence scoring regimes: raw XGBoost output, isotonic regression~\cite{zadrozny2002transforming}, and Platt scaling~\cite{platt1999probabilistic}. Full details are given in Appendix~\ref{expermentalData}.

Table~\ref{tab:report_real_datasets_averaged} and Appendix Table~\ref{tab:report_real} reports the results for the 15 datasets. Without post-hoc calibration, 6 out of 15 datasets exhibit risky confidence profiles, while the remaining datasets are safe with varying degrees of accuracy gain. Isotonic calibration dramatically worsens risk: 11 out of 15 datasets become risky, several with catastrophic Calibrated Size Ratios ($\mathrm{CSR} > 10^5$. This confirms that isotonic calibration, while effective at reducing ECE, can concentrate high confidence scores on incorrect predictions. This is a well-known problem, as top and bottom isotonic plateaus containing only positives or only negatives in the calibration set are assigned values of exactly 1 or 0. Platt scaling presents a different picture: only 3 datasets remain risky (\texttt{diabetes} , \texttt{breast\_cancer} and \texttt{wine}). These results illustrate that calibration methods can be orthogonal or even detrimental to risk. It is not surprising that standard methods fundamentally designed around minimizing ECE produce risky confidence profiles.  

Appendix Table~\ref{tab:alt_indicators} reports a complementary set of calibration and performance metrics with $\mathrm{ECE}_{15}$, Brier~\cite{brier1950} score, $\mathrm{ECD}$, $\mathrm{TCE}_{0.05}$, $\mathrm{T\text{-}CAL}^{1000}_{0.05}$, $\mathrm{MCalDist}$, $\mathrm{CWSA}_{0.5}$, and $\mathrm{CWSA}^+_{0.5}$ for the same datasets and calibration conditions.

\subsection{Sensitivity Analysis of $\mathrm{CSR}$}

Table~\ref{tab:prisk_extremes} is a summary of Tables~\ref{tab:report_all_indicators_A} and~\ref{tab:report_all_indicators_B} in the appendix, which report results across 10 confidence score distributions and 8 calibration modes, each repeated 100 times, for a total of 8,000 experiments. For each calibration mode, Table~\ref{tab:prisk_extremes} identifies the distribution that maximises and the distribution that minimises the average $P_{\mathrm{risk}}$, along with the corresponding percentage of trials exceeding 1 and 3 standard deviations. The table shows that $P_{\mathrm{risk}}$ reliably separates overconfident calibration modes (Overconf $1-\sqrt{1-c}$, Overconf $0.5c$, and Random under $c$) from well-calibrated and underconfident modes. The Random 0.5 calibration mode corresponds to random ground truth values independent of the confidence scores; under this mode, most confidence distributions are flagged as overconfident, with 7 out of 10 distributions reaching $P_{\mathrm{risk}}$ of 100\%. The three exceptions, Skew Low, Tight Low, and Log-Uniform Low, are not flagged as overconfident because their confidence scores are concentrated near 0, making them effectively underconfident relative to the Random 0.5 baseline, where the expected accuracy is 50\%. For well-calibrated and underconfident modes, $P_{\mathrm{risk}}$ remains low across all distributions, confirming that $\mathrm{CSR}$ does not raise false alarms under correct or conservative calibration.

\begin{table}[ht]
\centering
\footnotesize
\begin{tabular}{llrrr}
\toprule
\textbf{Calibration} & \textbf{Distribution} & $>1\boldsymbol{\sigma}_{\mathrm{CSR}}$ &  $>3\boldsymbol{\sigma}_{\mathrm{CSR}}$ & $\overline{\textbf{P}_{\textbf{risk}}}$ \\
\midrule
\multicolumn{5}{c}{\textit{Maximum $P_{\mathrm{risk}}$}} \\
\midrule
Random 0.5 & Uniform & 100.00\% & 100.00\% & 100.00\% \\
Perfect & Skew Low & 24.00\% & 0.00\% & 41.65\% \\
Underconf $0.2 + 0.8c$ & Tight Hi & 3.00\% & 0.00\% & 6.06\% \\
Underconf $\sqrt{c}$ & Skew High & 1.00\% & 1.00\% & 2.03\% \\
Random over c & Log-Uniform High & 1.00\% & 0.00\% & 2.80\% \\
Overconf  $1-\sqrt{(1-c)}$ & Normal & 100.00\% & 100.00\% & 100.00\% \\
Overconf $0.5c$  & Uniform & 100.00\% & 100.00\% & 100.00\% \\
Random under c & Tight Hi & 100.00\% & 100.00\% & 100.00\% \\
\midrule
\multicolumn{5}{c}{\textit{Minimum $P_{\mathrm{risk}}$}} \\
\midrule
Random 0.5 & Skew Low & 0.00\% & 0.00\% & 0.00\% \\
Perfect & Log-Uniform High & 1.00\% & 0.00\% & 9.32\% \\
Underconf $0.2 + 0.8c$ & Skew Low & 0.00\% & 0.00\% & 0.00\% \\
Underconf $\sqrt{c}$ & Skew Low & 0.00\% & 0.00\% & 0.00\% \\
Random over c & Uniform & 0.00\% & 0.00\% & 0.00\% \\
Overconf  $1-\sqrt{(1-c)}$ & Bimodal & 68.00\% & 30.00\% & 87.42\% \\
Overconf $0.5c$  & Log-Uniform Low & 100.00\% & 99.00\% & 100.00\% \\
Random under c & Skew High & 99.00\% & 99.00\% & 99.53\% \\
\bottomrule
\end{tabular}
\caption{Maximum and minimum average $P_{\mathrm{risk}}$ per calibration mode, with corresponding distribution, $>1\boldsymbol{\sigma}_{\mathrm{CSR}}$ and $>3\boldsymbol{\sigma}_{\mathrm{CSR}}$ counts, averaged over 100 repetitions. This is a summary of Tables ~\ref{tab:report_all_indicators_A} and~\ref{tab:report_all_indicators_B} in the Appendix.}
\label{tab:prisk_extremes}
\end{table}

\subsection{Confidence-weighted AUC}

Tables~\ref{tab:diffs_synthetic} and ~\ref{tab:full_raw_isotonic_and_platt} reports confidence-weighted accuracy and macro-averaged AUC for synthetic and real data respectively. Results show that $\text{cwA}$ is substantially greater than accuracy in all cases, synthetic and real. $\Delta\text{AUC}$, on the other hand, tells a more nuanced story: well-calibrated synthetic distributions show exclusively positive and often substantial increases  (see Table~\ref{tab:diffs_synthetic} in the appendix), while real datasets show smaller improvements and several negative values. This contrast is meaningful because, unlike accuracy, AUC already exploits confidence scores as a ranking signal.  The modest and sometimes negative $\Delta\text{AUC}$ values on real data admit two complementary explanations. The first is statistical: real datasets already achieve high AUC values, leaving little room for improvement. The second is more fundamental: the confidence scores produced by XGBoost are ultimately transformations of a discriminative score, using calibration methods that align their marginal distribution with observed accuracy. As a result, the transformed confidence scores may carry limited additional information beyond the original scores. 

For the synthetic distributions we see that calibrated and underconfident modes both show cwAUC $>$ AUC for 9 out of 10 distributions. Overconfident modes show mixed results, as inflated confidence scores preserve the ranking but mislead the probabilistic weighting. Random 0.5 mode, where ground truth labels are independent of confidence, yields cwAUC $<$ AUC for all ten distributions (see Table~\ref{tab:cwauc_summary} in the appendix).

Just as $\mathrm{cwA}$, the confidence-weighted AUC is blind to overconfidence risk. Table~\ref{tab:report_real_datasets_averaged} and Table~\ref{tab:full_raw_isotonic_and_platt} in the appendix shows the ordering $\mathrm{cwAUC}(\text{isotonic}) > \mathrm{cwAUC}(\text{raw}) > \mathrm{cwAUC}(\text{Platt})$ and  $\mathrm{cwA}(\text{isotonic}) > \mathrm{cwA}(\text{raw}) > \mathrm{cwA}(\text{Platt})$ which, together with $P_{\mathrm{risk}}(\text{isotonic}) > P_{\mathrm{risk}}(\text{raw}) > P_{\mathrm{risk}}(\text{Platt})$, indicates that isotonic and raw achieve better cwAUC and cwA at the cost of higher risk. 

\section{Limitations}

A few limitations deserve mention. First, CSR and $\sigma_{\mathrm{CSR}}$ involve terms sensitive to confidence scores near 1; although clipping confidences $\text{c} \rightarrow 1 - max(1 - \text{c}, \delta)$ when $N\delta  << 1$ seems to leave $P_{\mathrm{risk}}$ unaltered, a thorough analysis of numerical stability  is left for future work. Second, a non-risky CSR assessment is a statement about the limit distribution from which the evaluation set is assumed to be an i.i.d.\ sample. Once the model is deployed, this distribution may shift substantially, rendering $P_{\mathrm{risk}}$ misleading. Third, the validity of the Gaussian approximation for $P_{risk}$ at small N has not been studied in detail, though standard finite-sample corrections (e.g., Student-t, Edgeworth) apply. Fourth, we do not analyze whether CSR and the confidence-weighted metrics suffice to fully evaluate a confidence profile. Finally, whether cwA and CSR can guide post-hoc calibration, rather than merely evaluate it, has not been explored here and represents a natural direction for future work.

\section{Conclusion}

We propose two simple, directly interpretable metrics that complement each other in evaluating a model's confidence assignments: \[
 \text{CSR} = \mathbb{E}_{\mathcal{D}}\left[\frac{1-\mathbf{1}[\hat{y}_i=y_i]}{1-\text{conf}_i}\right] \quad \quad \text{and} \quad \quad \text{cwA} = \frac{\mathbb{E}_{\mathcal{D}}[\text{conf}_i \cdot \mathbf{1}[\hat{y}_i = y_i]]}{\mathbb{E}_{\mathcal{D}}[\text{conf}_i]}
\]
We introduce the {\bf{Calibrated Size Ratio (CSR)}}, which equals 1 under perfect calibration and has a $z$ score that {\bf{gives the probability that the confidence profile is risky}}. We tested CSR on real and synthetic distributions and found it separates safe from risky confidence profiles.

A confidence profile with acceptable risk can still be useful or useless, calibrated or uncalibrated. We show that the straightforward {\bf{Confidence Weighted Accuracy (cwA)}}, the fraction of total confidence mass assigned to correct predictions, {\bf{provides a directly interpretable measure of the usefulness of the confidences}}. Furthermore we show that the confidence-weighted approach  {\bf{extends naturally to all classical metrics including AUC}}, which has a confidence‑weighted counterpart that, unlike AUC, can measure the impact of calibration.

We also show that ECE is insensitive to risk. It is possible to construct arbitrarily risky models with small or even zero ECE, and experimental results on real datasets confirm that standard calibration methods give risky confidence assignments.

\newpage

\section*{Acknowledgments and Disclosure of Funding}

We gratefully acknowledge the funding and support of the Champalimaud Foundation (Lisbon, Portugal). We thank the Capgemini team in Portugal for providing computational resources through the Vodafone Test Bed project, in particular Filipa Grosso from the Technology and Innovation division, Luís Carlos Moreira from the Hybrid Intelligence division, and also to Antonio Ricciardo from Algebraic AI for coordination.

{

\bibliographystyle{unsrtnat}   
\bibliography{references}

}


\appendix

\section{Appendix}

\begin{definition}[Perfect Pointwise Calibration]\label{def:perfect_pointwise_calibration}
Let $\mathcal{D} = \{(x_i, y_i)\}_{i=1}^N$ be a dataset with $y_i \in \{1, \ldots, K\}$ for any $K \geq 2$, and let $\hat{y}_i$ be the model's predicted class. Let $\{\text{conf}_i\}_{i=1}^N$ be the set of confidence assignments.  We say the confidence profile is perfectly pointwise calibrated if: \[
P_{\mathcal{D}}(\hat{y}_i = y_i \mid \text{conf}_i = c) = c \quad \forall c \in [0,1] 
\]
that is, among all predictions assigned confidence $c$, the fraction that are correct is exactly $c$.
\end{definition}
\bigskip

This definition estimates the conditional probability empirically from $\mathcal{D}$ alone, i.e., as the observed fraction of correct predictions among all samples with confidence $c$. It assumes no other knowledge beyond  $\mathcal{D}$.  The conditional probability can be written as:
\[
P_{\mathcal{D}}(\hat{y}_i = y_i \mid \text{conf}_i = c) := \frac{\sum_{i:\text{conf}_i = c} \mathbf{1}[\hat{y}_i = y_i]}{\vert \{i: \text{conf}_i = c\} \vert} 
 \]
where $\mathbf{1}[\hat{y}_i = y_i]$ is the correctness indicator, equal to 1 if the predicted class matches the true class and 0 otherwise.

\bigskip

\subsection{ECE}

\bigskip
\begin{definition}[Empirical accuracy and confidence in a bin]
Let $\mathcal{D} = \{(x_i, y_i)\}_{i=1}^N$ and $\{\hat{y}_i, \text{conf}_i\}_{i=1}^N$ be as in Definition \ref{def:conf_assignment_setup}. For any bin $B \subseteq [0,1]$ such that $\{i : \mathrm{conf}_i \in B\} \neq \emptyset$, we define:

\begin{itemize}
    \item The \emph{empirical accuracy in bin $B$}:
    \[
        acc(B)
        :=
        \frac{1}{|B|}\sum_{i \in B} \mathbf{1}[\hat{y}_i = y_i]
    \]

    \item The \emph{empirical confidence in bin $B$}:
    \[
        conf(B)
        :=
        \frac{1}{|B|}\sum_{i \in B} \text{conf}_i.
    \]
\end{itemize}
\end{definition}

The classical definition of Expected Calibration Error (ECE), uses a partition of the confidence interval $[0,1]$ into $M$ bins, usually 15, and measures the weighted average of the discrepancy between empirical accuracy and mean confidence within each bin:
\bigskip

\begin{definition}[ECE]\label{def:ece_bins}
Let $\mathcal{D} = \{(x_i, y_i)\}_{i=1}^N$ and $\{\hat{y}_i, \text{conf}_i\}_{i=1}^N$ be as in Definition \ref{def:conf_assignment_setup}. Let $\{B_1, \ldots, B_M\}$ be a partition of the confidence interval  $[0,1]$ with equal-width bins $\{B_1, \ldots, B_M\}$ : \[
\text{ECE} = \sum_{m=1}^M \frac{|B_m|}{N} \left| \text{acc}(B_m) - \text{conf}(B_m) \right|. 
\]
\end{definition}

An equivalent way to write ECE is:
\bigskip

\begin{theorem}  \label{theorem:ECE_from_indicator}
Let $\mathcal{D} = \{(x_i, y_i)\}_{i=1}^N$ and $\{\hat{y}_i, \text{conf}_i\}_{i=1}^N$ be as in Definition \ref{def:conf_assignment_setup}, and let $\{B_1, \ldots, B_M\}$ be a partition of $[0,1]$ as in Definition \ref{def:ece_bins}. Then: \[
\text{ECE} = \frac{1}{N}\sum_{m=1}^M \left| \sum_{i \in B_m} \left(\mathbf{1}[\hat{y}_i = y_i] - \text{conf}_i\right) \right|
\]
\end{theorem}
\begin{proof}
Substituting the definitions of acc$(B_m)$ and conf$(B_m)$ into ECE: \[
\text{ECE} = \sum_{m=1}^M \frac{|B_m|}{N} \left| \text{acc}(B_m) - \text{conf}(B_m) \right| = \sum_{m=1}^M \frac{|B_m|}{N} \left| \frac{1}{|B_m|}\sum_{i \in B_m} \mathbf{1}[\hat{y}_i = y_i] - \frac{1}{|B_m|}\sum_{i \in B_m} \text{conf}_i \right| = 
\]
\[ = \sum_{m=1}^M \frac{|B_m|}{N} \cdot \frac{1}{|B_m|}\left| \sum_{i \in B_m} \left(\mathbf{1}[\hat{y}_i = y_i] - \text{conf}_i\right) \right| = \frac{1}{N}\sum_{m=1}^M \left| \sum_{i \in B_m} \left(\mathbf{1}[\hat{y}_i = y_i] - \text{conf}_i\right) \right| \]
\end{proof}
\bigskip

\subsection{Calibrated Size Ratio}

The condition $\text{conf}_i \in [0,1)$ is required for CSR to be well defined. Assigning $\text{conf}_i = 1$ to an incorrect prediction would make the denominator vanish, reflecting the fact that claiming absolute certainty on a wrong prediction is an infinitely bad error that yields an infinite CSR. Assigning $\text{conf}_i = 1$ to a correct prediction could be compatible with a defintion from CSR: \[
\text{CSR} = \frac{1}{N}\sum_{i:\hat{y}_i \neq y_i} \frac{1}{1 - \text{conf}_i}.
\] 
However as shown below, this convention leads to CSR $< 1$ even for epistemically honest assignments, and also causes problems with its variance, which is why we require $\text{conf}_i < 1$ for all $i$.
\bigskip

\begin{theorem} \label{theorem:PerfectCal_is_CSR_one}
Let $\mathcal{D} = \{(x_i, y_i)\}_{i=1}^N$ and $\{\hat{y}_i, \text{conf}_i\}_{i=1}^N$ be as in Definition \ref{def:conf_assignment_setup}.  Assume $\text{conf}_i < 1$ for all $i$. Perfect pointwise calibration implies CSR $ = 1$.
\end{theorem}
\begin{proof}
According to the defintion of CSR: \[
\text{CSR} = \frac{1}{N}\sum_{i:\hat{y}_i \neq y_i} \frac{1}{1-\text{conf}_i} = \sum_c P_{\mathcal{D}}(\text{conf}_i = c) \cdot \frac{P_{\mathcal{D}}(\hat{y}_i \neq y_i \mid \text{conf}_i = c)}{1-c} \]
Substituting perfect calibration $P_{\mathcal{D}}(\hat{y}_i \neq y_i \mid \text{conf}_i = c) = 1 - c$  for all $c \in [0,1)$:\[
= \sum_c P_{\mathcal{D}}(\text{conf}_i = c) \cdot \frac{1-c}{1-c} = \sum_c P_{\mathcal{D}}(\text{conf}_i = c) = 1
\]
where the condition $\text{conf}_i < 1$ ensures $1-c > 0$ so the cancellation is valid for all $c$ in the support.
\end{proof}
\bigskip

\begin{theorem} \label{theorem:PerfectCal_is_CSR_one_limitwise}
Let $\mathcal{D} = \{(x_i, y_i)\}_{i=1}^N$ and $\{\hat{y}_i, \text{conf}_i\}_{i=1}^N$ be as in Definition \ref{def:conf_assignment_setup}. Assume $\text{conf}_i < 1$ for all $i$. Perfect limitwise calibration implies CSR $ \approx 1$.
\end{theorem}
\begin{proof}
By perfect limitwise calibration, we have $P_{\text{limit}}(\hat{y}_i = y_i \mid \text{conf}_i = c) = c$  for every confidence value $c \in [0,1)$, which we can rewrite as
\[
\mathbb{E}_{limit}\big[\mathbf{1}[\hat{y}_i = y_i] \mid \text{conf}_i = c\big] = c.
\]
Perfect limitwise calibration assumes $\mathcal{D}$ correspond to i.i.d. draws from $P_{\text{limit}}$ so CSR can be approached by:
\[
\text{CSR} = \frac{1}{N}\sum_{i=1}^N \frac{1 - \mathbf{1}[\hat{y}_i = y_i]}{1 - \text{conf}_i} \approx \mathbb{E}_{limit}\left[\frac{1 - \mathbf{1}[\hat{y}_i = y_i]}{1 - \text{conf}_i}\right] 
\]
Conditioning on the confidence:
\[
\mathbb{E}_{limit}\left[\frac{1 - \mathbf{1}[\hat{y}_i = y_i]}{1 - \text{conf}_i}\right] = \mathbb{E}_{\text{conf}_i}\left[   \mathbb{E}_{limit}\left[\frac{1 - \mathbf{1}[\hat{y}_i = y_i]}{1 - \text{conf}_i}  \,\middle|\, \text{conf}_i\right] \right],
\]
where we used the law of total expectation (tower property).

Fix any value $\text{conf}_i = c \in [0,1)$.
By our assumption $\text{conf}_i < 1$, the denominator $1 - \text{conf}_i$ is strictly positive, so we can take it outside the inner conditional expectation:
\[
\mathbb{E}_{limit}\left[\frac{1 - \mathbf{1}[\hat{y}_i = y_i]}{1 - \text{conf}_i}
\,\middle|\, \text{conf}_i = c\right]
= 
\frac{1}{1 - c}\,
\mathbb{E}_{limit}\big[1 - \mathbf{1}[\hat{y}_i = y_i] \mid \text{conf}_i = c\big]
= \frac{1 - c}{1 - c} = 1.
\]
Substituting, $\text{CSR} \approx \mathbb{E}_{\text{conf}_i}\big[1\big]= 1$.
\end{proof}
\bigskip

CSR is equal to $1$ for perfectly calibrated confidence assignments, greater than $1$ for overconfident, and less than $1$ for underconfident. The same result also holds if we adopt a Bayesian instead of a frequentist approach:
\bigskip

\begin{theorem} \label{theorem:bin_size_zero_and_ECE_zero_is_CSR_one}
Let $\mathcal{D} = \{(x_i, y_i)\}_{i=1}^N$ and $\{\hat{y}_i, \text{conf}_i\}_{i=1}^N$ be as in Definition \ref{def:conf_assignment_setup}, and assume  $\text{conf}_i < 1$ for all $i$. As bin size tends to 0, ECE $= 0$ implies CSR $= 1$.
\end{theorem}
\begin{proof}
As bins become points ECE $ = 0$ becomes perfect pointwise calibration, so the result follows from Proposition \ref{theorem:PerfectCal_is_CSR_one}.
\end{proof}
\bigskip

The condition $\text{conf}_i < 1$ for all $i$ in the Proposition \ref{theorem:PerfectCal_is_CSR_one} and Proposition \ref{theorem:bin_size_zero_and_ECE_zero_is_CSR_one} deserves careful interpretation. Assigning $\text{conf}_i = 1$ to correct predictions seems harmless at first glance: correct predictions contribute nothing to CSR since $(1-y_i) = 0$, and a bin with conf $= 1$ and acc $= 1$ contributes nothing to ECE either. However both results fail in this case: an oracle that assigns conf $= 0$ to incorrect predictions and conf $= 1$ to correct predictions achieves perfect pointwise calibration yet CSR $= 1 - \bar{p} \neq 1$. We hence assume that all confidences must be strictly smaller than one. CSR $= 1$ and perfect knowledge of some samples requires the oracle to assign conf $= \bar{p}$ (the accuracy) to the remaining predictions (in this case all the incorrect ones) in order to recover CSR $= 1$. Assigning conf $= 0$ to incorrect predictions is interpreted by CSR as underconfidence: the model is claiming certainty of being wrong, yet the marginal probability of being correct is $\bar{p} > 0$.
\bigskip

\begin{theorem} \label{theorem:csr_and_ece_ortogonality}
Let $\mathcal{D} = \{(x_i, y_i)\}_{i=1}^N$ and  and $\{\hat{y}_i \}_{i=1}^N$ as in Definition \ref{def:conf_assignment_setup}, with at least one incorrect prediction. For any $\Lambda > 0$, there exists a confidence assignment $\{\text{conf}_i\}_{i=1}^N$ such that simultaneously: \[
\text{ECE} <  \frac{1 }{N}   \quad \text{and} \quad \text{CSR} > \Lambda. 
\]
\end{theorem}
\begin{proof}
 Choose one sample $i^*$ with $\hat{y}_{i^*} \neq y_{i^*}$ and extract it from $\mathcal{D}$.  It is always possible to assign confidences such that conf$(B_m) =$ acc$(B_m)$ for every bin and obtain a ECE $= 0$. Let $\{\text{conf}_i\}_{i=1; i \not= i^*}^N$ be such confidence profile for $N - 1$ samples. Let $K$ be the number of samples for which $c_i = 1$.  Now consider the set of $N$ confidences: \[
\text{conf}_i = \begin{cases} \min(c_i, 1 - \epsilon) & i \neq i^* \\ 1 - \delta & i = i^* \end{cases}
\]
where $\delta, \epsilon < 1/N$, so bins remain unchanged; any bin $B_m$ with $\text{acc}(B_m) < 1$ has some incorrect sample, hence $c_i = \text{acc}(B_m) \le 1 - 1/N$, so the clip $\min(c_i, 1-\epsilon)$ is inactive on such bins; and $\epsilon, \delta < 1/N \le 1/M$ keeps $1-\epsilon$ and $1-\delta$ in the top bin. Thus no sample changes bin, and perfect calibration is preserved in every bin not containing $i^*$. Using Proposition \ref{theorem:ECE_from_indicator} to write ECE and taking into account the perfect calibration within each bin $\sum_{i \in B_m} \left(\mathbf{1}[\hat{y}_i = y_i] - \text{conf}_i\right) = 0$, the value of ECE with the injected sample
\[
\text{ECE} = \frac{1}{N}\sum_{m=1}^M \left| \sum_{i \in B_m} \left(\mathbf{1}[\hat{y}_i = y_i] - \text{conf}_i\right) \right| =  \frac{K\epsilon + 1 - \delta }{N}  
\]
If $K \not= 0$, choose $\epsilon < \min(\delta/K, 1/N) =\delta/K$. In any case:
\[
\text{ECE} =  \frac{K\epsilon + 1 - \delta }{N}  <  \frac{1 }{N}  
\]
CSR is well defined as $\text{conf}_i < 1$ for all $i$. The contribution of sample $i^*$ to CSR is: \[
\text{CSR} \geq \frac{1}{N} \cdot \frac{1}{1-(1-\delta)} = \frac{1}{N\delta} \to \infty \text{ as } \delta \to 0 
\]
So for any $\Lambda > 0$, choose $\delta < \frac{1}{N \Lambda}$ to guarantee $CSR > \Lambda$.
\end{proof}
\bigskip

\begin{theorem} \label{theorem:std_csr_under_perfect_limitwise}
Let $\mathcal{D} = \{(x_i, y_i)\}_{i=1}^N$ and $\{\hat{y}_i, \text{conf}_i\}_{i=1}^N$ be as in Definition \ref{def:conf_assignment_setup} Assume $\text{conf}_i < 1$ for all $i$ and perfect limitwise calibration. Then, the standard deviation of CSR is: \[
\sigma_{\text{CSR}} = \sqrt{\frac{1}{N}\mathbb{E}_{limit} \left[\frac{\text{conf}_i}{1-\text{conf}_i}\right]} \approx   \sqrt{\frac{1}{N^2}\sum_{i=1}^N \frac{\text{conf}_i}{1-\text{conf}_i}}  \]
\end{theorem}
\begin{proof}
Write CSR as an average of random variables i.i.d. $Z_i = \frac{1-\mathbf{1}[\hat{y}_i=y_i]}{1-\text{conf}_i}$: \[
\text{CSR} = \frac{1}{N}\sum_{i=1}^N Z_i \implies \text{Var}(\text{CSR}) = \frac{1}{N^2}\sum_{i=1}^N \text{Var}(Z_i) =  \frac{1}{N} \text{Var}(Z) \]
Conditioning on $\text{conf}_i = c$ and using perfect limitwise calibration $P_{\text{limit}}(\hat{y}_i = y_i \mid \text{conf}_i = c) = c$, the variable $Z$ takes value $0$ with probability $c$ and $\frac{1}{1-c}$ with probability $1-c$. Its variance is: \[
\text{Var}(Z \mid \text{conf}_i = c) = \frac{1}{(1-c)^2}\cdot \text{Var}(1-\mathbf{1}[\hat{y}_i=y_i]) = \frac{1}{(1-c)^2}\cdot c(1-c) = \frac{c}{1-c} 
\] where we have used the Bernoulli variance. We now use the law of total variance: \[
\mathrm{Var}(Z) = \mathbb{E}\bigl[\mathrm{Var}(Z \mid \text{conf})\bigr] + \mathrm{Var}\bigl(\mathbb{E}[Z \mid \text{conf}]\bigr).
\]
The second summand is 0 because $\mathbb{E}[Z \mid \text{conf} =c] = 1$ so the unconditional variance is: \[
\text{Var}(Z) = \sum_c  \text{Var}(Z \mid \text{conf} = c) p(c) = \mathbb{E}_{limit}   \left[\frac{\text{conf}_i}{1-\text{conf}_i}\right]
\] If we approximate this expectation by the empirical mean over the $N$ examples, we obtain
\[
\mathrm{Var}(\text{CSR}) =  \frac{1}{N} \mathbb{E}_{limit}   \left[\frac{\text{conf}_i}{1-\text{conf}_i}\right]  \approx \frac{1}{N^2}\sum_{i=1}^N \frac{\text{conf}_i}{1 - \text{conf}_i}. \]
Taking the square root  $\sigma_{\text{CSR}} = \sqrt{\text{Var}(\text{CSR})  }$ we have the standard deviation.
\end{proof}
\bigskip

Note that while the CSR depends solely on incorrect assignments, its variance is determined by the odds of all assignments within the dataset.
\bigskip

\begin{theorem} \label{proposition:jensen_CSR}
Let $\mathcal{D} = \{(x_i, y_i)\}_{i=1}^N$ and $\{\hat{y}_i, \mathrm{conf}_i\}_{i=1}^N$ be as in 
Definition~\ref{def:conf_assignment_setup}, and let $\bar{p}$ denote the empirical accuracy. 
If $\mathrm{conf}_i \in [0,1)$ for all $i$, then
\[
    \mathrm{CSR} \geq \frac{1 - \bar{p}}{1 - \mathbb{E}[\mathrm{conf}_i \mid \hat{y}_i \neq y_i]}.
\]
\end{theorem}

\begin{proof}
The conditional expectation of confidence on incorrect predictions is
\[
    \mathbb{E}_{\mathcal{D}}[\mathrm{conf}_i \mid \hat{y}_i \neq y_i] 
    = \frac{\sum_{i=1}^N \mathrm{conf}_i \cdot \mathbf{1}[\hat{y}_i \neq y_i]}
           {\sum_{i=1}^N \mathbf{1}[\hat{y}_i \neq y_i]},
\]
where the denominator satisfies
\[
    \sum_{i=1}^N \mathbf{1}[\hat{y}_i \neq y_i] = N(1 - \bar{p}).
\]
Since $f(x) = \frac{1}{1-x}$ is strictly convex on $[0,1)$ and $\mathrm{conf}_i \in [0,1)$, 
Jensen's inequality gives $f(\mathbb{E}[X \mid \hat{y}_i \neq y_i]) \leq \mathbb{E}[f(X) \mid \hat{y}_i \neq y_i]$, 
that is:
\[
    \frac{1}{1 - \mathbb{E}_{\mathcal{D}}[\mathrm{conf}_i \mid \hat{y}_i \neq y_i]}
    \leq \mathbb{E}_{\mathcal{D}}\!\left[\frac{1}{1 - \mathrm{conf}_i} \,\middle|\, \hat{y}_i \neq y_i \right]
    = \frac{1}{N(1-\bar{p})} \sum_{i:\,\hat{y}_i \neq y_i} \frac{1}{1 - \mathrm{conf}_i}
    = \frac{\mathrm{CSR}}{1 - \bar{p}}.
\]
Rearranging yields the result.
\end{proof}
\bigskip

\begin{theorem} \label{proposition:_average_of_confidences for_incorrect_under_CSR_one}
Let $\mathcal{D} = \{(x_i, y_i)\}_{i=1}^N$ and $\{\hat{y}_i, \text{conf}_i\}_{i=1}^N$ be as in Definition \ref{def:conf_assignment_setup} and let $\bar{p}$ be the empirical accuracy.  If CSR = $ 1$ and $\text{conf}_i \in [0,1)$ for all $i$, then  \[
\mathbb{E}_{\mathcal{D}}[\text{conf}_i \mid \hat{y}_i \neq y_i] :=  \frac{\sum_{i=1}^N  conf_i \cdot \mathbf{1}[\hat{y}_i \neq y_i]}{\sum_{i=1}^N  \mathbf{1}[\hat{y}_i \neq y_i]} \leq \bar{p}. 
\] 
\end{theorem}
\begin{proof}
From Proposition \ref{proposition:jensen_CSR}, using CSR $ = 1$:  \[
\frac{1 - \bar{p}}{1- \mathbb{E}_{\mathcal{D}}[\text{conf}_i \mid \hat{y}_i \neq y_i]  } \leq  1
\]
and then $\mathbb{E}[\text{conf}_i \mid \hat{y}_i \neq y_i ]  \leq   \bar{p}.$
\end{proof}
\bigskip

\subsection{Confidence-Weighted Metrics}

\begin{theorem}[cwA as Covariance] \label{prop:cwA_covariance}
Let $\mathcal{D} = \{(x_i, y_i)\}_{i=1}^N$ and $\{\hat{y}_i, \mathrm{conf}_i\}_{i=1}^N$ be as in 
Definition~\ref{def:conf_assignment_setup}. Then:
\[
    \mathrm{cwA} = \bar{p} +  \frac{\mathrm{Cov}_{\mathcal{D}}(\mathrm{conf}_i,\, \mathbf{1}[\hat{y}_i = y_i])}{\bar{c}}
\]
where $\bar{c} = \mathbb{E}_{\mathcal{D}}[\mathrm{conf}_i]$ is the mean confidence and $\bar{p}$ is the empirical accuracy.
\end{theorem}
\begin{proof}
By definition of covariance:
\[
    \mathrm{Cov}_{\mathcal{D}}(\mathrm{conf}_i,\, \mathbf{1}[\hat{y}_i = y_i]) 
    = \mathbb{E}_{\mathcal{D}}[\mathrm{conf}_i \cdot \mathbf{1}[\hat{y}_i = y_i]] - \mathbb{E}_{\mathcal{D}}[\mathrm{conf}_i] \cdot \mathbb{E}_{\mathcal{D}}[\mathbf{1}[\hat{y}_i = y_i]] = 
\]\[
	= \mathbb{E}_{\mathcal{D}}[\mathrm{conf}_i \cdot \mathbf{1}[\hat{y}_i = y_i]] - \bar{c} \bar{p}
\]
From the Definition~\ref{def:cwA}:
\[
    \mathrm{cwA} = \frac{\mathbb{E}_{\mathcal{D}}[\mathrm{conf}_i \cdot \mathbf{1}[\hat{y}_i = y_i]]}{\mathbb{E}_{\mathcal{D}}[\mathrm{conf}_i]} = \frac{\mathrm{Cov}_{\mathcal{D}}(\mathrm{conf}_i,\, \mathbf{1}[\hat{y}_i = y_i]) + \bar{c}\bar{p}}{\bar{c}}. 
\]
\end{proof}

\begin{corollary} \label{corollary:calibrated_yet_uselss}
Let $\mathcal{D} = \{(x_i, y_i)\}_{i=1}^N$ and $\{\hat{y}_i, \text{conf}_i\}_{i=1}^N$ be as in Definition \ref{def:conf_assignment_setup}. Let $\bar{p}$ be the accuracy. The flat confidence assignment: \[
\text{conf}_i = \bar{p} \quad \forall i \]
simultaneously gives perfect pointwise calibration and is completely uninformative, i.e. $\mathrm{Cov}_{\mathcal{D}}(\mathrm{conf}_i,\, \mathbf{1}[\hat{y}_i = y_i]) = 0$.
\end{corollary}
\begin{proof}
The only confidence value in the support is $\bar{p}$, and by definition of empirical accuracy: \[
P_{\mathcal{D}}(\hat{y}_i = y_i \mid \text{conf}_i = \bar{p}) = \bar{p} 
\] so the assignment is perfectly calibrated. In this case it also holds cwA $= \bar{p}$, so from \ref{prop:cwA_covariance} the $\mathrm{Cov}_{\mathcal{D}}(\mathrm{conf}_i,\, \mathbf{1}[\hat{y}_i = y_i]) = \bar{c}(cwA - \bar{p}) = 0$.
\end{proof}
\bigskip

\begin{theorem} \label{proposition:cw_basic_relations}
For all $k \in \{1, \ldots, K\}$: \[
\text{cwP}^{(k)} = \text{cwTP}^{(k)} + \text{cwFN}^{(k)} \quad \quad \text{cwN}^{(k)} = \text{cwFP}^{(k)} + \text{cwTN}^{(k)}
\]   \[
cwP^{(k)} + cwN^{(k)} = N\bar{c}    \quad \quad   \text{cwTP}^{(k)} + \text{cwFP}^{(k)} + \text{cwFN}^{(k)} + \text{cwTN}^{(k)} = N\bar c \quad \forall k.
\] 
\end{theorem}
\begin{proof}
\[\text{cwTP}^{(k)} + \text{cwFN}^{(k)} = \sum_{i=1}^N \text{conf}_i^{(\hat{y}_i)} \cdot \mathbf{1}[\hat{y}_i = y_i] \cdot \mathbf{1}[y_i=k] + \sum_{i=1}^N \text{conf}_i^{(\hat{y}_i)} \cdot \mathbf{1}[\hat{y}_i \neq y_i] \cdot \mathbf{1}[y_i=k] \] \[= \sum_{i=1}^N \text{conf}_i^{(\hat{y}_i)} \cdot \mathbf{1}[y_i=k] \cdot \underbrace{(\mathbf{1}[\hat{y}_i = y_i] + \mathbf{1}[\hat{y}_i \neq y_i])}_{=1} = \sum_{i=1}^N \text{conf}_i^{(\hat{y}_i)} \cdot \mathbf{1}[y_i=k] = cwP^{(k)}  \] 
\[
\text{cwFP}^{(k)} + \text{cwTN}^{(k)} = \sum_{i=1}^N \text{conf}_i \cdot \mathbf{1}[y_i \neq k]\cdot \mathbf{1}[\hat y_i = k] \;+\; \sum_{i=1}^N \text{conf}_i \cdot \mathbf{1}[y_i \neq k]\cdot \mathbf{1}[\hat y_i \neq k] \]\[
= \sum_{i=1}^N \text{conf}_i \cdot \mathbf{1}[y_i \neq k] \cdot \underbrace{\big(\mathbf{1}[\hat y_i = k] + \mathbf{1}[\hat y_i \neq k]\big)}_{=\,1} \;=\; \sum_{i=1}^N \text{conf}_i \cdot \mathbf{1}[y_i \neq k] \;=\; \text{cwN}^{(k)}. \]
 \[
\text{cwP}^{(k)} + \text{cwN}^{(k)} = \sum_{i=1}^N \text{conf}_i^{(\hat{y}_i)} \cdot (\mathbf{1}[y_i=k] + \mathbf{1}[y_i\neq k]) = \sum_{i=1}^N \text{conf}_i^{(\hat{y}_i)} = N\bar{c} 
\] \[
\text{cwTP}^{(k)} + \text{cwFP}^{(k)} + \text{cwFN}^{(k)} + \text{cwTN}^{(k)} = \underbrace{\big(\text{cwTP}^{(k)} + \text{cwFN}^{(k)}\big)}_{=\,\text{cwP}^{(k)}} + \underbrace{\big(\text{cwFP}^{(k)} + \text{cwTN}^{(k)}\big)}_{=\,\text{cwN}^{(k)}}  \]\[
= \text{cwP}^{(k)} + \text{cwN}^{(k)} = N\bar c. \]
\end{proof}
\bigskip

\begin{theorem} \label{theorem:generalized_metrics}
For any metric $M$ that is an algebraic function of $\text{TP}^{(k)}, \text{FP}^{(k)}, \text{FN}^{(k)}, \text{TN}^{(k)}$, define its confidence-weighted generalization as: \[
\text{cw}M^{(k)} = M(\text{cwTP}^{(k)}, \text{cwFP}^{(k)}, \text{cwFN}^{(k)}, \text{cwTN}^{(k)}) 
\]
Then, all algebraic identities between classical metrics that can be derived from identities between $\text{TP}^{(k)}, \text{FP}^{(k)}, \text{FN}^{(k)}, \text{TN}^{(k)}$, also hold for their confidence-weighted counterparts.
\end{theorem}
\begin{proof}
The only identities that relate the classic $\text{TP}^{(k)}, \text{FP}^{(k)}, \text{FN}^{(k)}, \text{TN}^{(k)}$ are the four given in Proposition \ref{proposition:cw_basic_relations}.  According to this proposition, the confidence-weighted counterparts of TP, FP, FN, TN satisfy the same four structural relations than the classic $\text{TP}^{(k)}, \text{FP}^{(k)}, \text{FN}^{(k)}, \text{TN}^{(k)}$, so any algebraic identity that holds due to structural identities of $\text{TP}^{(k)}, \text{FP}^{(k)}, \text{FN}^{(k)}, \text{TN}^{(k)}$ should also hold for the confidence-weighted metrics.
\end{proof}
\bigskip

\begin{theorem}  \label{cwa_from_standard_metrics}
\label{theorem:cwA_sum_of_TP}
The overall confidence-weighted accuracy equals the total per-class
true-positive mass, normalized by the total confidence mass:
\[
\text{cwA} \;=\; \frac{\sum_{k=1}^K \text{cwTP}^{(k)}}{N\bar c}
       \;=\; \frac{\sum_{k=1}^K \text{cwTP}^{(k)}}{\sum_{k=1}^K \text{cwP}^{(k)}}.
\]
\end{theorem}
\begin{proof}
Using $\sum_{k=1}^K \mathbf{1}[y_i = k] = 1$: \[
\sum_{k=1}^K \text{cwTP}^{(k)} = \sum_{k=1}^K \sum_{i=1}^N \text{conf}_i \cdot \mathbf{1}[\hat y_i = y_i]\cdot \mathbf{1}[y_i = k] = \sum_{i=1}^N \text{conf}_i \cdot \mathbf{1}[\hat y_i = y_i]\cdot\underbrace{\sum_{k=1}^K \mathbf{1}[y_i = k]}_{=\,1} 
\] \[
= \sum_{i=1}^N \text{conf}_i \cdot \mathbf{1}[\hat y_i = y_i] \;=\; N\bar c \cdot \text{cwA}. 
\]
Dividing by $N\bar c$ yields the first equality. For the second, $\sum_{k=1}^K \text{cwP}^{(k)} = \sum_i \text{conf}_i \sum_k \mathbf{1}[y_i=k] = \sum_i \text{conf}_i = N\bar c$.
\end{proof}
\bigskip

\begin{theorem} \label{theorem:cwA_macro}
The macro-average of the per-class confidence-weighted accuracies is an
affine function of the overall confidence-weighted accuracy:
\[
\sum_{k=1}^K \text{cwAcc}^{(k)} \;=\; (K-2) \;+\; 2\,\text{cwA}
\]
In particular, for $K=2$ we recover $\text{cwAcc}^{(1)} = \text{cwAcc}^{(2)} = \text{cwA}$.
\end{theorem}
\begin{proof}
By Proposition \ref{proposition:cw_basic_relations}, $\text{cwP}^{(k)} + \text{cwN}^{(k)} = N\bar c$, so \[
\sum_{k=1}^K \text{cwAcc}^{(k)} = \sum_{k=1}^K \frac{\text{cwTP}^{(k)} + \text{cwTN}^{(k)}}{N\bar c} = \frac{1}{N\bar c}\sum_{k=1}^K \big(\text{cwTP}^{(k)} + \text{cwTN}^{(k)}\big). 
\] Expanding each term using the standard definitions: \[
\sum_{k=1}^K\!\big(\text{cwTP}^{(k)} + \text{cwTN}^{(k)}\big) = \sum_{k=1}^K\sum_{i=1}^N \text{conf}_i\Big(\mathbf{1}[\hat y_i=k]\,\mathbf{1}[y_i=k] + \mathbf{1}[\hat y_i\neq k]\,\mathbf{1}[y_i\neq k]\Big). 
\] Swap the order of summation and analyze the inner sum over $k$ for a fixed sample $i$.  Suppose $\hat y_i = y_i$. Exactly one value of $k$ (namely $k = y_i$) satisfies $\hat y_i = k$ and $y_i = k$, contributing $1$ to the first indicator product. For the other $K-1$ values of $k$, both $\hat y_i \neq k$ and $y_i \neq k$, contributing $1$ each to the second indicator product. Total: $1 + (K-1) = K$. Now suppose $\hat y_i \neq y_i$. For $k = y_i$: $\hat y_i \neq k$ but $y_i = k$, contributes $0$. For $k = \hat y_i$: $\hat y_i = k$ but $y_i \neq k$, contributes $0$. For the remaining $K-2$ values of $k$, $\hat y_i \neq k$ and $y_i \neq k$, each contributes $1$. Total: $K-2$.
Therefore, letting $M := \sum_i \text{conf}_i\,\mathbf{1}[\hat y_i = y_i]$ denote the total confidence mass on correct predictions: \[
\sum_{k=1}^K\!\big(\text{cwTP}^{(k)} + \text{cwTN}^{(k)}\big) = K\cdot M + (K-2)\cdot(N\bar c - M) = (K-2)\,N\bar c + 2M. 
\] Since $\text{cwA} = M / (N\bar c)$, dividing by $N\bar c$ and using $\text{cwA} = M/(N\bar c)$: \[
\sum_{k=1}^K \text{cwAcc}^{(k)} = (K-2) + 2\,\text{cwA}. \]
The last statement for $K=2$, follows as $\text{cwAcc}^{(1)} = \text{cwAcc}^{(2)}$ with sum that is equal to $2\,\text{cwA}$. 
\end{proof}
\smallskip

{\bf{Remark}}:  Setting $\text{conf}_i \equiv 1$ recovers the classical identity $\sum_{k=1}^K \text{Acc}^{(k)} = (K-2) + 2\,\text{A}$, with $A$ the overall accuracy, which quantifies the well-known TN-inflation of macro-averaged per-class accuracy in multiclass classification. Proposition \ref{theorem:cwA_macro} shows that confidence weighting preserves this exact affine structure. \\
Since $K-2 \geq 0$ for $K \geq 2$, and $\text{cwA} \leq 1$: \[
\frac{1}{K}\sum_{k=1}^K \text{cwAcc}^{(k)} \;\geq\; \text{cwA}, 
\] with equality iff $K=2$ or $\text{cwA}=1$. Macro-averaged per-class accuracy is always an optimistic summary compared to overall accuracy when $K\geq 3$.  
\bigskip

The per-class accuracies $\text{Acc}^{(k)}$ are notably absent from standard evaluation pipelines. Proposition \ref{theorem:cwA_macro} explains why: their macro-average is an affine function of $\text{A} = \bar{p}$ alone, \[
\frac{1}{K}\sum_{k=1}^K \text{Acc}^{(k)} = \frac{K-2}{K} + \frac{2}{K}\,\text{A}, 
\] and is therefore dominated, for large $K$, by the combinatorial floor $(K-2)/K$, a quantity unrelated to classifier quality that reflects only the growing proportion of true negatives in each one-vs-rest decomposition. As $K\to\infty$, this floor tends to $1$ while the signal term $2\,\text{A}/K$ vanishes, so macro-averaged per-class accuracy converges to $1$ uniformly over all classifiers with bounded $\text{A}$, losing any ability to discriminate between good and bad models. The same pathology applies verbatim to $\text{cwAcc}^{(k)}$, with $\text{A}$ replaced by $\text{cwA}$: confidence weighting does not remove the TN-inflation, it merely reweights the signal. For this reason, practitioners report $\text{A}$ (or now $\text{cwA}$) directly, and rely on TN-free metrics (recall, precision, $F_1$, or balanced accuracy) when class-specific diagnostics are required.
\bigskip

The classical AUC has the well known probabilistic interpretation:  \begin{equation}  \label{auc_equation_appendix}
\text{AUC}^{(k)} =  P\left(\text{conf}^{(k)}(x^+) > \text{conf}^{(k)}(x^-)\right) + \frac{1}{2}P\left(\text{conf}^{(k)}(x^+) = \text{conf}^{(k)}(x^-)\right)
\end{equation}
where $x^+$ and $x^-$ are drawn uniformly from $\mathcal{D}^+_k$ and $\mathcal{D}^-_k$ respectively. An analogous result also holds for the confidence‑weighted version of AUC:
\smallskip

\begin{theorem} \label{thorem:AUC}
Let $\mathcal{D} = \{(x_i, y_i)\}_{i=1}^N$ with $y_i \in \{1, \ldots, K\}$, let $\hat{y}_i$ be the model's predicted class and let $\text{conf}_i^{(k)}$ be the confidence assigned to class $k$ for sample $i$. For each class $k$, define: \[
\text{cwTPR}^{(k)}(\tau) = \frac{\text{cwTP}^{(k)}(\tau)}{P^{(k)}} = \frac{\sum_{i:y_i=k, \text{conf}_i^{(k)}\geq\tau} \text{conf}_i^{(\hat{y}_i)}}{\sum_{i:y_i=k} \text{conf}_i^{(\hat{y}_i)}} 
\]\[
\text{cwFPR}^{(k)}(\tau) = \frac{\text{cwFP}^{(k)}(\tau)}{N^{(k)}} = \frac{\sum_{i:y_i\neq k, \text{conf}_i^{(k)}\geq\tau} \text{conf}_i^{(\hat{y}_i)}}{\sum_{i:y_i\neq k} \text{conf}_i^{(\hat{y}_i)}} \]
and the cw-AUC as the area under the cw-ROC curve: \[
\text{cwAUC}^{(k)} = \int_0^1 \text{cwTPR}^{(k)}(\tau) \, d\text{cwFPR}^{(k)}(\tau) 
\] Then: \[
\text{cwAUC}^{(k)} = P_w\left(\text{conf}^{(k)}(x^+) > \text{conf}^{(k)}(x^-)\right) + \frac{1}{2}P_w\left(\text{conf}^{(k)}(x^+) = \text{conf}^{(k)}(x^-)\right)
\]
where $x^+$ and $x^-$ are drawn from $\mathcal{D}^+_k$ and $\mathcal{D}^-_k$ proportionally to $\text{conf}^{(\hat{y})}$: \[
P_w(i,j) = \frac{\text{conf}_i^{(\hat{y}_i)} \cdot \text{conf}_j^{(\hat{y}_j)}}{\sum_{i' \in \mathcal{D}^+_k}\sum_{j' \in \mathcal{D}^-_k} \text{conf}_{i'}^{(\hat{y}_{i'})} \cdot \text{conf}_{j'}^{(\hat{y}_{j'})}} 
\]
\end{theorem}
\begin{proof}
Each step $\Delta\text{cwFPR}(\tau)$ corresponds to a negative sample $j$ with $\tau = \text{conf}_j^{(k)}$ contributing: \[
\Delta\text{cwFPR}(\tau) = \frac{\text{conf}_j^{(\hat{y}_j)} \cdot  \mathbf{1}[y_j \neq k] }{\sum_{j \in \mathcal{D}^-_k}  \text{conf}_j^{(\hat{y}_j)}} 
\] At this step: \[ 
\text{cwTPR}(\tau) = \frac{ \frac{1}{2} \sum_{i:y_i=k,\, \text{conf}_i^{(k)} \geq \tau}\text{conf}_i^{(\hat{y}_i)} + \frac{1}{2}\sum_{i:y_i=k,\, \text{conf}_i^{(k)} > \tau}\text{conf}_i^{(\hat{y}_i)}}{\sum_{i \in \mathcal{D}^+_k} \text{conf}_i^{(\hat{y}_i)}}
\] Therefore: \[
\text{cwAUC}^{(k)} = \sum_{j \in \mathcal{D}} \frac{\text{conf}_j^{(\hat{y}_j)} \cdot  \mathbf{1}[y_j \neq k]  }{\sum_{j \in \mathcal{D}^-_k}  \text{conf}_j^{(\hat{y}_j)}} \cdot \frac{ \sum_{i:y_i=k,\, \text{conf}_i^{(k)} > \tau}\text{conf}_i^{(\hat{y}_i)} + \frac{1}{2}\sum_{i:y_i=k,\, \text{conf}_i^{(k)} = \tau}\text{conf}_i^{(\hat{y}_i)}  }{\sum_{i \in \mathcal{D}^+_k} \text{conf}_i^{(\hat{y}_i)}}  =
\] \[
 = \frac{\sum_{i \in \mathcal{D}^+_k}\sum_{j \in \mathcal{D}^-_k} \text{conf}_i^{(\hat{y}_i)} \cdot \text{conf}_j^{(\hat{y}_j)} \cdot \left(\mathbf{1}[\text{conf}_i^{(k)} > \text{conf}_j^{(k)}] + \frac{1}{2}\mathbf{1}[\text{conf}_i^{(k)} = \text{conf}_j^{(k)}]\right)}{\sum_{i \in \mathcal{D}^+_k}\sum_{j \in \mathcal{D}^-_k} \text{conf}_i^{(\hat{y}_i)} \cdot \text{conf}_j^{(\hat{y}_j)}} =
\]\[
= P_w\left(\text{conf}^{(k)}(x^+) > \text{conf}^{(k)}(x^-)\right) + \frac{1}{2}P_w\left(\text{conf}^{(k)}(x^+) = \text{conf}^{(k)}(x^-)\right)
\]
\end{proof}
\bigskip

{\bf{Remark}}:  Notice that $i \in \mathcal{D}^+_k$ and $j \in \mathcal{D}^-_k$ does not imply $\hat{y}_i = k$ or $\hat{y}_j \not= k$.
The classical AUC, the area under the ROC curve, is a purely ordinal metric that  measures whether positives are ranked above negatives. A correct ranking made with confidence $0.51$ contributes identically to one made with confidence $0.99$.
The cw-AUC replaces uniform sampling with confidence-weighted sampling proportional to $\text{conf}_i^{(\hat{y}_i)} \cdot \text{conf}_j^{(\hat{y}_j)}$.
Pairs where both samples receive high predicted-class confidence are upweighted. The model is rewarded for correct rankings it makes confidently, and penalized less for incorrect rankings it makes uncertainly. cwAUC tend to be numerically higher than the classical AUC;  A model that is more confident when it ranks correctly than when it ranks incorrectly will score higher under cwAUC. 
\bigskip

\begin{theorem} \label{thorem:AUC_diffs}
Let $\mathcal{D} = \{(x_i, y_i)\}_{i=1}^N$ with $y_i \in \{1, \ldots, K\}$, let $\hat{y}_i$ be the model's predicted class and let $\text{conf}_i^{(k)}$ be the confidence assigned to class $k$ as in Proposition \ref{thorem:AUC}. For each class $k$, define for each pair $(i,j)$ with $i \in \mathcal{D}^+_k$, $j \in \mathcal{D}^-_k$  (i.e, $y_i = k$ and $y_j \not= k$): \[
z_{ij} =\mathbf{1}[\text{conf}_i^{(k)} > \text{conf}_j^{(k)}] + \frac{1}{2}\mathbf{1}[\text{conf}_i^{(k)} = \text{conf}_j^{(k)}] \quad \text{and} \quad w_{ij} = \text{conf}_i^{(\hat{y}_i)} \cdot \text{conf}_j^{(\hat{y}_j)} 
\]  Then: \[
\text{cwAUC}^{(k)} - \text{AUC}^{(k)} = \frac{\text{Cov}(w_{ij}, z_{ij})}{\mathbb{E}[w_{ij}]}
\]
\end{theorem}
\begin{proof}
With this notation: \[
\text{cwAUC}^{(k)} = \frac{\mathbb{E}[w_{ij} z_{ij}]}{\mathbb{E}[w_{ij}]} \quad \text{AUC}^{(k)} = \mathbb{E}[z_{ij}]
\]
Therefore: \[
\text{cwAUC}^{(k)} - \text{AUC}^{(k)} = \frac{\mathbb{E}[w_{ij} z_{ij}]}{\mathbb{E}[w_{ij}]} - \mathbb{E}[z_{ij}] = \frac{\mathbb{E}[w_{ij} z_{ij}] - \mathbb{E}[w_{ij}]\mathbb{E}[z_{ij}]}{\mathbb{E}[w_{ij}]} = \frac{\text{Cov}(w_{ij}, z_{ij})}{\mathbb{E}[w_{ij}]} \]
\end{proof}
\bigskip

{\bf{Remark}}:   The gap cwAUC $-$ AUC $= \frac{\text{Cov}(w_{ij}, z_{ij})}{\mathbb{E}[w_{ij}]}$ is positive when the model is more confident on pairs it ranks correctly than on pairs it ranks incorrectly, zero when confidence is independent of ranking quality, and negative when the model is more confident on incorrectly ranked pairs. A model with high cwAUC $-$ AUC is not only discriminative but also knows when it is being reliably discriminative. 
\bigskip

Furthermore, we can show the following;

\begin{theorem}[AUC invariance to monotone rescaling]  \label{cor:auc_invariance}
Let $\phi : \mathbb{R} \to \mathbb{R}$ be any strictly monotone increasing function. Then:
\[
\text{AUC}^{(k)}\!\left(\phi \circ \text{conf}^{(k)}\right) = \text{AUC}^{(k)}\!\left(\text{conf}^{(k)}\right)
\]
\end{theorem}
\begin{proof}
Setting $\text{conf}_i^{(\hat{y}_i)} = 1$ for all i in Proposition~\ref{thorem:AUC} recovers classical AUC. In this case every step $\Delta\text{FPR}(\tau) = 1/N^{(k)}$ is uniform and independent of the score values. The threshold condition $\text{conf}_i^{(k)} \geq \tau$
depends only on the ordering of scores, which $\phi$ preserves by monotonicity. Therefore AUC is unchanged. An alternative proof follows from Equation~\ref{auc_equation_appendix}, whose right-hand side is a function solely of the pairwise orderings $\mathbf{1}[\text{conf}^{(k)}(x_i) > \text{conf}^{(k)}(x_j)]$. Since $\phi$ is strictly increasing, $\phi(a) > \phi(b) \iff a > b$, so every such indicator is unchanged, and the sum is therefore identical.
\end{proof}
\bigskip

{\bf{Remark}}: Since any strictly monotone post-hoc calibration method leaves the ranking of $\text{conf}^{(k)}$ unchanged, $\text{AUC}^{(k)}$ is invariant to calibration. By contrast, $\text{cwAUC}^{(k)}$ is sensitive to the magnitudes of $\text{conf}^{(\hat{y})}$ through the pairwise weights $w_{ij}$. Therefore $\text{cwAUC}^{(k)} - \text{AUC}^{(k)}$ measures whatever discriminative value is added by calibration. \\
Two common deviations from this setting are worth noting: \\
Isotonic regression is only weakly monotone. The fitted map is piecewise constant, so distinct raw scores within a plateau are collapsed to a single calibrated value. Pairs $(x_i, x_j)$ that were strictly ordered before calibration become ties afterwards, and the strict-monotonicity hypothesis of the theorem fails. Under the standard $+\tfrac{1}{2}$ tie-breaking convention this typically induces a small but nonzero shift in $\text{AUC}^{(k)}$. \\
Per-sample renormalization in the multiclass case. Libraries such as \texttt{scikit-learn}'s \texttt{CalibratedClassifierCV} fit a separate calibrator $\phi_k$ per class and then renormalize per sample so that outputs sum to one: \[
\tilde{p}_k(x) \;=\; \frac{\phi_k(s_k(x))}{\sum_{j=1}^{K} \phi_j(s_j(x))}. 
\] The denominator depends on the scores of all other classes, so the mapping $s_k \mapsto \tilde{p}_k$ is no longer a function of $s_k$ alone and the theorem does not apply.
These two effects compound: under isotonic calibration with multiclass renormalization, the ties introduced by plateaus are resolved by the other classes' calibrated scores rather than by class $k$'s own ranking, which can produce noticeably larger deviations in $\text{AUC}^{(k)}$ than either effect alone. Invariance is recovered exactly when $\phi$ is strictly increasing (e.g., Platt scaling) and applied per class without renormalization.

\subsection{Preparation of experimental data} \label{expermentalData}

We evaluate on fifteen classification datasets, ten binary and five multiclass, drawn from the UCI Machine Learning Repository~\cite{dua2019uci}, CIFAR-10~\cite{krizhevsky2009learning}, and MNIST~\cite{lecun1998gradient}. Binary tasks include UCI Adult Income  \cite{uci_adult_income_1996}, Bank Marketing~\cite{uci_bank_marketing_2014}, Breast Cancer Wisconsin~\cite{uci_breast_cancer_1988}, CERN Higgs Boson~\cite{uci_higgs_2014}, Credit Card Default~\cite{uci_credit_approval_1987}, Pima Indians Diabetes~\cite{uci_diabetes_uci_aim94}, Indian Liver Patient Records~\cite{uci_liver_disorders_2016}, Phishing Websites~\cite{uci_phishing_websites_2012}, NATICUSdroid Android Permissions~\cite{uci_naticusdroid_2021}, and an Accelerometer Gyro Mobile Phone activity dataset~\cite{uci_accelerometer_2012}. Multiclass tasks include CIFAR-10, Covertype~\cite{uci_covertype_1998}, Obesity Levels~\cite{uci_obesity_2019}, Wine Quality~\cite{uci_wine_1992}, and MNIST.

For each dataset, a deterministic stratified 60/20/20 split produces training, validation, and test partitions. When an official train/test split is available (UCI Adult, MNIST, CIFAR-10), the official test set is retained and a validation set is carved from the training data, yielding an effective 75/25 train/validation split. All splits are performed before any shuffling, so the test partition is held out entirely from training, hyperparameter selection, and calibration.

All datasets are modelled with XGBoost~\cite{chen2016xgboost} using the histogram tree method. Binary tasks use a logistic objective and multiclass tasks use a softmax probability objective. Hyperparameters are selected by Bayesian optimisation with the Tree-structured Parzen Estimator~\cite{bergstra2011algorithms} as implemented in Optuna~\cite{akiba2019optuna}, with 60 trials evaluated via stratified 5-fold cross-validation and XGBoost early stopping (patience 30 rounds, maximum 2000 boosting rounds per fold).

Class probabilities are obtained natively from XGBoost \cite{chen2016xgboost,pedregosa2011scikitlearn}. The boosted-tree margins are mapped through the logistic sigmoid to yield $P(y = 1 | x)$ and normalised by softmax to produce a categorical distribution. These native probabilities are taken as the raw (uncalibrated) confidences. Two post-hoc calibration methods are fitted on the held-out validation set: isotonic regression~\cite{zadrozny2002transforming} and Platt scaling~\cite{platt1999probabilistic}. Results are therefore reported under three confidence scoring regimes: raw XGBoost output, isotonic regression, and Platt scaling, allowing direct comparison of how calibration affects the indicators of interest. All confidence scores are clipped to $[\varepsilon, 1-\varepsilon]$ with $\varepsilon = 10^{-8}$ prior to computation of metrics to avoid numerical overflows. 

Wine Quality Exception: Two rare classes in this dataset (quality 3 and 9) have counts in the held-out validation set too low to satisfy  \texttt{CalibratedClassifierCV}'s internal requirement of at least two samples per class per fold. Classes with too few examples were left uncalibrated. All other 14 datasets underwent regular calibration.

\subsection{Computational Resources}

Most of the experiments described in this work are computationally lightweight and reproducible on a standard desktop or laptop in seconds to a few minutes per dataset. The XGBoost training, the 60-trial Optuna hyperparameter sweep with 5-fold cross-validation complete on a single GPU (NVIDIA GeForce RTX 3050) for the small/medium tabular datasets (Adult, Bank Marketing, Breast Cancer, Credit, Diabetes, Liver, Phishing, NATICUSdroid, Accelerometer, Obesity, MNIST, and Wine) with peak GPU memory below 4 GB.  For the three more demanding datasets (CIFAR-10, Covertype, and the CERN Higgs Boson challenge) the Optuna sweep is appreciably heavier (the feature dimension and/or sample count are larger). These two sweeps were executed on a larger workstation: two servers with NVIDIA Tesla T4 GPU's, on which each experiment run completes in under 24 hours. All scripts use the same code path and configuration regardless of hardware; the only difference is wall-clock time. A single random seed is fixed across every experiment so that results are bit-reproducible on equivalent hardware.

\subsection{Additional tables and figures} \label{app:alt_tables}

\begin{table}[h]
\centering
\footnotesize
\begin{tabular}{lll}
\toprule
\textbf{Distribution} & \textbf{Sampling Procedure} & \textbf{Description} \\
\midrule
Uniform & $c \sim \mathcal{U}(0, 1)$ & Scores spread uniformly across $[0,1]$ \\
Skew High & $c \sim \text{Beta}(3.0,\ 0.5)$ & Scores concentrated near 1 \\
Skew Low & $c \sim \text{Beta}(0.5,\ 3.0)$ & Scores concentrated near 0 \\
Bimodal & $c \sim \tfrac{1}{2}\text{Beta}(0.5,\ 3.0) + \tfrac{1}{2}\text{Beta}(3.0,\ 0.5)$ & Scores concentrated near both extremes \\
Tight Hi & $c \sim \mathcal{U}(0.8,\ 1.0)$ & Scores confined to a narrow high range \\
Tight Lo & $c \sim \mathcal{U}(0.0,\ 0.2)$ & Scores confined to a narrow low range \\
Normal & $c \sim \mathcal{N}_{[0,1)}(0.7,\ 0.1^2)$ & Most scores in $(0.4,1)$, support $[0,1)$ \\
Log-Uniform Low & $c = \exp(u),\ u \sim \mathcal{U}(\log 10^{-4},\ \log(1 - 10^{-6}))$ & Density skewed toward 0 \\
Log-Uniform High & $c = 1 - \exp(u),\ u \sim \mathcal{U}(\log 10^{-6},\ \log 0.9)$ & Density skewed toward 1 \\
Bell & $c \sim \text{Beta}(5.0,\ 5.0)$ & Symmetrically clustered around $0.5$ \\
\midrule
\textbf{Mode} & $p_{\text{true}}(c)$ & \textbf{Description} \\
\midrule
Random 0.5        & $0.5$                        & Fully uninformative \\
Perfect           & $c$                          & Identity; perfectly calibrated \\
Underconf & $0.2 + 0.8c$            & Linear compression toward center \\
Underconf  & $\sqrt{c}$               & Concave push toward 1 \\
Random over $c$  & $u \sim \mathcal{U}(c,\ 1)$ & Random draw above confidence score \\
Overconf  & $1-\sqrt{(1-c)}$ & Convex pull toward 0 \\
Overconf *0.5 & $0.5c$               & Hard scaling toward 0 \\
Random under $c$  & $u \sim \mathcal{U}(0,\ c)$ & Random draw below confidence score \\
\bottomrule
\end{tabular}
\caption{Synthetic confidence score distributions and calibration modes defining the map from confidence score $c$ to true label probability $p_{\text{true}}(c)$. }
\label{tab:distributions}
\end{table}

\begin{table}[ht]
\centering
\footnotesize
\begin{tabular}{llrrrrrrrr}
\toprule
\textbf{Distribution} & $\overline{\textbf{Acc}}$ & $\overline{\textbf{cwA}}$ & $\overline{\textbf{gain}}$ & $\overline{\textbf{CSR}}$ & $\overline{\boldsymbol{\sigma}_{\mathrm{CSR}}}$ & $>1\boldsymbol{\sigma}_{\mathrm{CSR}}$ &  $>3\boldsymbol{\sigma}_{\mathrm{CSR}}$ & $\overline{\textbf{P}_{\textbf{risk}}}$ \\
\midrule
\multicolumn{9}{c}{\textit{N = 100  (avg over 100 reps)}} \\
\midrule
Uniform  & 0.4987 & 0.6626 & 32.87\% & 1.0251 & 0.3108 & 13.00\% & 3.00\% & 29.84\% \\
Skew High  & 0.8662 & 0.8956 & 21.99\% & 0.9115 & 11.6795 & 2.00\% & 0.00\% & 16.16\% \\
Skew Low  & 0.1456 & 0.3361 & 22.43\% & 0.9985 & 0.0500 & 15.00\% & 0.00\% & 32.53\% \\
Bimodal  & 0.4970 & 0.8084 & 62.04\% & 1.2019 & 7.8461 & 6.00\% & 2.00\% & 17.90\% \\
Tight Hi  & 0.9039 & 0.9075 & 3.91\% & 0.9638 & 0.6288 & 10.00\% & 1.00\% & 26.46\% \\
Tight Lo  & 0.0991 & 0.1324 & 3.75\% & 1.0004 & 0.0339 & 19.00\% & 0.00\% & 37.30\% \\
Normal  & 0.6947 & 0.7087 & 4.61\% & 1.0182 & 0.1920 & 17.00\% & 1.00\% & 35.81\% \\
Log-Uniform Low  & 0.1103 & 0.4940 & 43.39\% & 0.9918 & 0.0783 & 7.00\% & 0.00\% & 28.53\% \\
Log-Uniform High  & 0.9358 & 0.9620 & 41.75\% & 0.4842 & 27.1426 & 0.00\% & 0.00\% & 3.69\% \\
Bell  & 0.4950 & 0.5407 & 9.12\% & 1.0124 & 0.1120 & 21.00\% & 0.00\% & 37.83\% \\
\midrule
\multicolumn{9}{c}{\textit{N = 10,000  (avg over 100 reps)}} \\
\midrule
Uniform  & 0.5006 & 0.6671 & 33.33\% & 0.9993 & 0.0375 & 9.00\% & 1.00\% & 32.75\% \\
Skew High  & 0.8571 & 0.8889 & 22.25\% & 1.0297 & 16.1439 & 1.00\% & 1.00\% & 17.72\% \\
Skew Low  & 0.1423 & 0.3329 & 22.22\% & 1.0006 & 0.0050 & 18.00\% & 0.00\% & 42.14\% \\
Bimodal  & 0.5002 & 0.8096 & 61.91\% & 1.0072 & 3.6306 & 1.00\% & 0.00\% & 16.72\% \\
Tight Hi  & 0.8997 & 0.9034 & 3.69\% & 1.0056 & 0.0790 & 12.00\% & 1.00\% & 33.00\% \\
Tight Lo  & 0.1002 & 0.1332 & 3.67\% & 0.9999 & 0.0034 & 14.00\% & 0.00\% & 37.56\% \\
Normal  & 0.6986 & 0.7128 & 4.72\% & 1.0026 & 0.0190 & 21.00\% & 2.00\% & 40.26\% \\
Log-Uniform Low  & 0.1087 & 0.4989 & 43.78\% & 1.0012 & 0.0106 & 14.00\% & 1.00\% & 37.99\% \\
Log-Uniform High  & 0.9342 & 0.9611 & 40.89\% & 0.9608 & 2.7029 & 2.00\% & 1.00\% & 10.22\% \\
Bell  & 0.5003 & 0.5457 & 9.08\% & 0.9990 & 0.0112 & 14.00\% & 0.00\% & 31.20\% \\
\midrule
\multicolumn{9}{c}{\textit{N = 100,000  (avg over 100 reps)}} \\
\midrule
Uniform  & 0.5002 & 0.6668 & 33.33\% & 1.0002 & 0.0120 & 6.00\% & 1.00\% & 24.70\% \\
Skew High  & 0.8571 & 0.8888 & 22.19\% & 1.0279 & 10.1119 & 1.00\% & 0.00\% & 16.79\% \\
Skew Low  & 0.1428 & 0.3330 & 22.19\% & 1.0001 & 0.0016 & 16.00\% & 0.00\% & 39.93\% \\
Bimodal  & 0.5000 & 0.8094 & 61.87\% & 0.9920 & 5.7997 & 0.00\% & 0.00\% & 15.79\% \\
Tight Hi  & 0.8999 & 0.9036 & 3.70\% & 1.0035 & 0.0266 & 11.00\% & 1.00\% & 40.24\% \\
Tight Lo  & 0.1001 & 0.1335 & 3.71\% & 0.9999 & 0.0011 & 10.00\% & 0.00\% & 33.32\% \\
Normal  & 0.6997 & 0.7139 & 4.70\% & 0.9991 & 0.0058 & 5.00\% & 0.00\% & 33.31\% \\
Log-Uniform Low  & 0.1085 & 0.4998 & 43.90\% & 0.9999 & 0.0037 & 7.00\% & 0.00\% & 30.83\% \\
Log-Uniform High  & 0.9344 & 0.9613 & 41.09\% & 0.9409 & 0.8537 & 6.00\% & 1.00\% & 11.39\% \\
Bell  & 0.5000 & 0.5454 & 9.09\% & 1.0001 & 0.0035 & 15.00\% & 0.00\% & 36.37\% \\
\bottomrule
\end{tabular}
\caption{Results on calibrated synthetic datasets, varying dataset size ($N$), averaged over 100 repetitions. $>1\boldsymbol{\sigma}_{\mathrm{CSR}}$ and $>3\boldsymbol{\sigma}_{\mathrm{CSR}}$ give the percentage of repetitions in which $\mathrm{CSR}$ exceeds 1 by more than 1 (risk of 84.13\%) and more than 3 (risk of 99.87\% ) standard deviations respectively.}
\label{tab:report_perfect_synthetic}
\end{table}

\begin{table}[ht]
\centering
\small
\label{tab:report_2}
\begin{tabular}{lrrrrrr}
\toprule
\textbf{Dataset} & $\textbf{Acc}$ & $\textbf{cwA}$ & $\textbf{gain}$ & $\textbf{CSR}$ & $\boldsymbol{\sigma}_{\mathrm{CSR}}$ & $\textbf{P}_{\textbf{risk}}$ \\
\midrule
\multicolumn{7}{c}{\textit{no post-hoc calibration}} \\
\midrule
accelerometer & 0.9755 & 0.9807 & 21.1\% & 3.3572 & 0.8817 & 99.62\% \\
adult\_income & 0.8408 & 0.8698 & 18.2\% & 0.9517 & 0.1337 & 0.00\% \\
bank\_marketing & 0.9052 & 0.9236 & 19.4\% & 0.9169 & 0.4018 & 0.00\% \\
breast\_cancer & 0.9737 & 0.9763 & 9.8\% & 1.1729 & 2.9505 & 52.34\% \\
cifar10 & 0.5858 & 0.6673 & 19.7\% & 1.1880 & 0.0360 & 100.00\% \\
covertype & 0.9587 & 0.9693 & 25.6\% & 0.6249 & 0.5146 & 0.00\% \\
credit & 0.7948 & 0.8126 & 8.7\% & 1.3495 & 0.0419 & 100.00\% \\
diabetes & 0.6818 & 0.7119 & 9.5\% & 1.2439 & 0.1903 & 90.00\% \\
higgs\_boson & 0.8352 & 0.8606 & 15.4\% & 1.0708 & 0.0262 & 99.66\% \\
liver & 0.7350 & 0.7600 & 9.4\% & 0.8773 & 0.2034 & 0.00\% \\
mnist & 0.9765 & 0.9835 & 29.9\% & 0.3475 & 0.6893 & 0.00\% \\
naticusdroid & 0.9724 & 0.9768 & 15.8\% & 0.8651 & 0.7855 & 0.00\% \\
obesity & 0.9811 & 0.9861 & 26.6\% & 0.2114 & 0.6376 & 0.00\% \\
phishing & 0.9729 & 0.9792 & 23.5\% & 1.5214 & 10.2345 & 52.03\% \\
wine & 0.6385 & 0.6778 & 10.9\% & 1.1482 & 0.0534 & 99.72\% \\
\midrule
\multicolumn{7}{c}{\textit{isotonic}} \\
\midrule
accelerometer & 0.9755 & 0.9854 & 40.6\% & 31256.2721 & 74.4730 & 100.00\% \\
adult\_income & 0.8408 & 0.8928 & 32.6\% & 33201.5178 & 37.6852 & 100.00\% \\
bank\_marketing & 0.9052 & 0.9373 & 33.8\% & 0.6415 & 66.7860 & 0.00\% \\
breast\_cancer & 0.9737 & 0.9732 & -1.9\% & 2631578.9341 & 924.1801 & 100.00\% \\
cifar10 & 0.5858 & 0.6724 & 20.9\% & 1.0801 & 3.7418 & 50.85\% \\
covertype & 0.9587 & 0.9684 & 23.7\% & 2582.7239 & 10.3478 & 100.00\% \\
credit & 0.7948 & 0.8283 & 16.3\% & 1.0567 & 1.6670 & 51.36\% \\
diabetes & 0.6818 & 0.7388 & 17.9\% & 649352.0315 & 267.7342 & 100.00\% \\
higgs\_boson & 0.8352 & 0.8660 & 18.7\% & 2001.0498 & 6.2706 & 100.00\% \\
liver & 0.7350 & 0.7938 & 22.2\% & 854704.2294 & 452.2652 & 100.00\% \\
mnist & 0.9765 & 0.9836 & 30.3\% & 1.0226 & 71.1763 & 50.01\% \\
naticusdroid & 0.9724 & 0.9771 & 17.1\% & 34089.5707 & 92.1192 & 100.00\% \\
obesity & 0.9811 & 0.9858 & 24.7\% & 472813.2757 & 444.1677 & 100.00\% \\
phishing & 0.9729 & 0.9805 & 28.3\% & 180913.8123 & 190.6580 & 100.00\% \\
wine & 0.6385 & 0.6778 & 10.9\% & 1.1482 & 0.0534 & 99.72\% \\
\midrule
\multicolumn{7}{c}{\textit{platt}} \\
\midrule
accelerometer & 0.9755 & 0.9842 & 35.5\% & 0.7763 & 0.1364 & 0.00\% \\
adult\_income & 0.8408 & 0.8830 & 26.5\% & 0.9246 & 0.0410 & 0.00\% \\
bank\_marketing & 0.9052 & 0.9338 & 30.1\% & 0.6893 & 0.0700 & 0.00\% \\
breast\_cancer & 0.9737 & 0.9747 & 3.8\% & 2.4935 & 1.1177 & 90.93\% \\
cifar10 & 0.5858 & 0.6525 & 16.1\% & 0.9184 & 0.0136 & 0.00\% \\
covertype & 0.9587 & 0.9660 & 17.8\% & 0.6255 & 0.0229 & 0.00\% \\
credit & 0.7948 & 0.8260 & 15.2\% & 1.0025 & 0.0325 & 53.11\% \\
diabetes & 0.6818 & 0.7303 & 15.2\% & 1.1998 & 0.1878 & 85.63\% \\
higgs\_boson & 0.8352 & 0.8599 & 15.0\% & 0.9525 & 0.0150 & 0.00\% \\
liver & 0.7350 & 0.7802 & 17.0\% & 0.7508 & 0.1606 & 0.00\% \\
mnist & 0.9765 & 0.9792 & 11.7\% & 0.5692 & 0.0742 & 0.00\% \\
naticusdroid & 0.9724 & 0.9748 & 8.7\% & 0.6511 & 0.1033 & 0.00\% \\
obesity & 0.9811 & 0.9839 & 14.8\% & 0.1845 & 0.2238 & 0.00\% \\
phishing & 0.9729 & 0.9775 & 17.2\% & 0.6198 & 0.1745 & 0.00\% \\
wine & 0.6385 & 0.6778 & 10.9\% & 1.1482 & 0.0534 & 99.72\% \\
\bottomrule
\end{tabular}
\caption{Results on real datasets without post-hoc calibration and with isotonic and Platt scaling post-hoc calibration on a separated held-out set. Many datasets exhibit risky confidence profiles, particularly after isotonic calibration.}
\label{tab:report_real}
\end{table}

\begin{table}[ht]
\centering
\scriptsize
\label{tab:report_4}
\begin{tabular}{lrrrrrrrrrr}
\toprule
dataset & ECE$_{15}$ & Brier & ECD  & ->test & TCE$_{0.05}$  & T-CAL$^{1000}_{0.05}$  & ->reject & MCalDist & CWSA$_{0.5}$ & CWSA$^{+}_{0.5}$ \\
\midrule
\multicolumn{11}{c}{\textit{no post-hoc calibration}} \\
\midrule
accelerometer & 0.0089 & 0.0175 & 0.0251 & 1 & 46.7261 & 0.0001 & 1 & 0.0091 & 0.9417 & 0.9553 \\
adult\_income & 0.0303 & 0.1066 & 0.0238 & 1 & 47.6627 & 0.0006 & 1 & 0.0311 & 0.5972 & 0.6632 \\
bank\_marketing & 0.0289 & 0.0668 & 0.0553 & 1 & 35.3727 & 0.0008 & 1 & 0.0302 & 0.7724 & 0.8203 \\
breast\_cancer & 0.0254 & 0.0267 & 0.0161 & 0 & 0.0000 & 0.0039 & 0 & 0.0120 & 0.8974 & 0.9171 \\
cifar10 & 0.0388 & 0.1925 & 0.0220 & 1 & 60.2700 & 0.0015 & 1 & 0.0391 & 0.3264 & 0.4205 \\
covertype & 0.0223 & 0.0322 & -0.0468 & -1 & 83.8825 & 0.0005 & 1 & 0.0226 & 0.8416 & 0.8580 \\
credit & 0.0318 & 0.1528 & 0.0425 & 1 & 45.9500 & 0.0005 & 1 & 0.0318 & 0.4138 & 0.5095 \\
diabetes & 0.0747 & 0.1952 & 0.0559 & 0 & 4.5455 & 0.0048 & 0 & 0.0717 & 0.2685 & 0.3800 \\
higgs\_boson & 0.0112 & 0.1157 & 0.0160 & 1 & 44.7580 & 0.0001 & 1 & 0.0119 & 0.5498 & 0.6210 \\
liver & 0.1072 & 0.1810 & -0.0063 & 0 & 0.0000 & 0.0112 & 0 & 0.0542 & 0.2751 & 0.3540 \\
mnist & 0.0025 & 0.0159 & -0.0008 & 0 & 7.4600 & 0.0003 & 0 & 0.0040 & 0.9435 & 0.9531 \\
naticusdroid & 0.0062 & 0.0232 & 0.0032 & 0 & 6.4258 & 0.0010 & 1 & 0.0045 & 0.9085 & 0.9260 \\
obesity & 0.0304 & 0.0207 & -0.0498 & -1 & 4.4917 & 0.0019 & 0 & 0.0304 & 0.8961 & 0.9037 \\
phishing & 0.0058 & 0.0202 & 0.0093 & 0 & 8.7291 & 0.0019 & 1 & 0.0072 & 0.9210 & 0.9344 \\
wine & 0.0323 & 0.2051 & 0.0264 & 1 & 16.8462 & 0.0087 & 1 & 0.0411 & 0.2364 & 0.3386 \\
\midrule
\multicolumn{11}{c}{\textit{isotonic}} \\
\midrule
accelerometer & 0.0041 & 0.0136 & 0.0125 & 1 & 32.7864 & 0.0001 & 0 & 0.0055 & 0.9628 & 0.9718 \\
adult\_income & 0.0119 & 0.0916 & 0.0176 & 1 & 50.7105 & 0.0003 & 1 & 0.0120 & 0.7211 & 0.7679 \\
bank\_marketing & 0.0073 & 0.0571 & 0.0039 & 0 & 16.2418 & 0.0001 & 0 & 0.0077 & 0.8188 & 0.8455 \\
breast\_cancer & 0.0452 & 0.0426 & 0.6596 & 1 & 35.9649 & 0.0081 & 1 & 0.0438 & 0.9415 & 0.9683 \\
cifar10 & 0.0142 & 0.1909 & 0.0064 & 0 & 13.9500 & 0.0002 & 0 & 0.0173 & 0.3194 & 0.4002 \\
covertype & 0.0022 & 0.0302 & -0.0019 & -1 & 43.0738 & 0.0000 & 1 & 0.0041 & 0.8801 & 0.8997 \\
credit & 0.0155 & 0.1384 & 0.0063 & 0 & 18.4000 & 0.0005 & 1 & 0.0213 & 0.4952 & 0.5843 \\
diabetes & 0.0778 & 0.1840 & 0.1987 & 1 & 25.3247 & 0.0021 & 0 & 0.0826 & 0.4328 & 0.5326 \\
higgs\_boson & 0.0041 & 0.1117 & 0.0041 & 1 & 32.7320 & 0.0001 & 0 & 0.0069 & 0.5888 & 0.6555 \\
liver & 0.1019 & 0.1716 & 0.1792 & 1 & 7.6923 & 0.0055 & 0 & 0.0854 & 0.4215 & 0.4913 \\
mnist & 0.0037 & 0.0158 & 0.0052 & 1 & 9.8900 & 0.0004 & 0 & 0.0036 & 0.9483 & 0.9580 \\
naticusdroid & 0.0058 & 0.0233 & 0.0091 & 1 & 16.9763 & 0.0006 & 1 & 0.0065 & 0.9135 & 0.9305 \\
obesity & 0.0282 & 0.0225 & 0.0711 & 1 & 29.3144 & 0.0024 & 1 & 0.0269 & 0.9183 & 0.9268 \\
phishing & 0.0108 & 0.0198 & 0.0304 & 1 & 33.7404 & 0.0009 & 0 & 0.0103 & 0.9312 & 0.9432 \\
wine & 0.0323 & 0.2051 & 0.0264 & 1 & 16.8462 & 0.0087 & 1 & 0.0411 & 0.2364 & 0.3386 \\
\midrule
\multicolumn{11}{c}{\textit{platt}} \\
\midrule
accelerometer & 0.0038 & 0.0141 & -0.0007 & 0 & 61.3846 & 0.0006 & 1 & 0.0124 & 0.9535 & 0.9642 \\
adult\_income & 0.0194 & 0.0947 & 0.0065 & 0 & 70.5445 & 0.0002 & 1 & 0.0309 & 0.6792 & 0.7377 \\
bank\_marketing & 0.0173 & 0.0589 & 0.0003 & 0 & 84.2437 & 0.0007 & 1 & 0.0337 & 0.8108 & 0.8474 \\
breast\_cancer & 0.0215 & 0.0278 & 0.0281 & 0 & 0.0000 & 0.0023 & 0 & 0.0137 & 0.8957 & 0.9186 \\
cifar10 & 0.0647 & 0.1972 & -0.0331 & -1 & 63.4600 & 0.0004 & 1 & 0.0639 & 0.2186 & 0.3014 \\
covertype & 0.0093 & 0.0313 & -0.0044 & -1 & 99.7048 & 0.0000 & 1 & 0.0203 & 0.8816 & 0.9065 \\
credit & 0.0153 & 0.1382 & -0.0002 & 0 & 14.4167 & 0.0007 & 1 & 0.0205 & 0.4867 & 0.5772 \\
diabetes & 0.0679 & 0.1794 & 0.0300 & 0 & 0.0000 & 0.0031 & 0 & 0.0634 & 0.3490 & 0.4494 \\
higgs\_boson & 0.0266 & 0.1139 & 0.0034 & 0 & 77.7020 & 0.0002 & 1 & 0.0294 & 0.5738 & 0.6525 \\
liver & 0.1111 & 0.1735 & -0.0394 & 0 & 3.4188 & 0.0153 & 0 & 0.0988 & 0.3606 & 0.4387 \\
mnist & 0.0095 & 0.0192 & -0.0148 & -1 & 94.8300 & 0.0006 & 1 & 0.0280 & 0.9220 & 0.9393 \\
naticusdroid & 0.0090 & 0.0243 & -0.0035 & 0 & 80.1261 & 0.0006 & 1 & 0.0205 & 0.9076 & 0.9292 \\
obesity & 0.0422 & 0.0178 & -0.1080 & -1 & 22.2222 & 0.0013 & 1 & 0.0392 & 0.8650 & 0.8765 \\
phishing & 0.0080 & 0.0206 & -0.0123 & 0 & 46.7662 & 0.0017 & 1 & 0.0158 & 0.9163 & 0.9330 \\
wine & 0.0323 & 0.2051 & 0.0264 & 1 & 16.8462 & 0.0087 & 1 & 0.0411 & 0.2364 & 0.3386 \\
\bottomrule
\end{tabular}
\caption{Various other common metrics including the calibration test for ECD (with 0 for calibrated, -1 for underconfident and 1 for overconfident, using significance level of 0.05 and 5000 bootstrap simulations) and the adaptive T-CAL$^{1000}_{0.05}$ calibration reject test (1 for reject, 0 for accept, with 0.05 significance level and 1000 Monte Carlo simulations per bin level) for the same datasets without post-hoc calibration and with isotonic and Platt scaling post-hoc calibration.}
\label{tab:alt_indicators}
\end{table}

\begin{table}[ht]
\centering
\tiny
\begin{tabular}{llrrrrrrrr}
\toprule
\textbf{Distribution} & \textbf{Calibration} & $\overline{\textbf{Acc}}$ & $\overline{\textbf{cwA}}$ & $\overline{\textbf{gain}}$ & $\overline{\textbf{CSR}}$ & $\overline{\boldsymbol{\sigma}_{\mathrm{CSR}}}$ & $>1\boldsymbol{\sigma}_{\mathrm{CSR}}$ &  $>3\boldsymbol{\sigma}_{\mathrm{CSR}}$ & $\overline{\textbf{P}_{\textbf{risk}}}$ \\
\midrule
\multicolumn{10}{c}{\textit{random}} \\
\midrule
Uniform & Random 0.5 & 0.5007 & 0.5001 & -0.13\% & 7.2495 & 0.1045 & 100.00\% & 100.00\% & 100.00\% \\
Skew High & Random 0.5 & 0.5032 & 0.5027 & -0.08\% & 138863.1123 & 88.6478 & 100.00\% & 100.00\% & 100.00\% \\
Skew Low & Random 0.5 & 0.5027 & 0.5014 & -0.25\% & 0.6216 & 0.0157 & 0.00\% & 0.00\% & 0.00\% \\
Bimodal & Random 0.5 & 0.5000 & 0.4997 & -0.06\% & 296252.2417 & 5.6867 & 100.00\% & 100.00\% & 100.00\% \\
Tight Hi & Random 0.5 & 0.5010 & 0.5009 & -0.00\% & 37.2159 & 0.2336 & 100.00\% & 100.00\% & 100.00\% \\
Tight Lo & Random 0.5 & 0.4987 & 0.4991 & 0.07\% & 0.5591 & 0.0107 & 0.00\% & 0.00\% & 0.00\% \\
Normal & Random 0.5 & 0.4982 & 0.4985 & 0.05\% & 2.0802 & 0.0560 & 100.00\% & 100.00\% & 100.00\% \\
Log-Uniform Low & Random 0.5 & 0.4997 & 0.4987 & -0.21\% & 0.9792 & 0.0300 & 22.00\% & 19.00\% & 25.58\% \\
Log-Uniform High & Random 0.5 & 0.5029 & 0.5028 & -0.01\% & 36163.8773 & 8.5918 & 100.00\% & 100.00\% & 100.00\% \\
Bell & Random 0.5 & 0.4978 & 0.4973 & -0.10\% & 1.1266 & 0.0352 & 99.00\% & 63.00\% & 99.18\% \\
\midrule
\multicolumn{10}{c}{\textit{perfect}} \\
\midrule
Uniform & Perfect & 0.4977 & 0.6638 & 33.08\% & 1.0062 & 0.1045 & 6.00\% & 1.00\% & 32.52\% \\
Skew High & Perfect & 0.8571 & 0.8886 & 22.04\% & 1.0249 & 88.6478 & 1.00\% & 1.00\% & 13.23\% \\
Skew Low & Perfect & 0.1405 & 0.3274 & 21.76\% & 1.0032 & 0.0157 & 24.00\% & 0.00\% & 41.65\% \\
Bimodal & Perfect & 0.5000 & 0.8094 & 61.90\% & 0.9768 & 5.6867 & 0.00\% & 0.00\% & 16.19\% \\
Tight Hi & Perfect & 0.9006 & 0.9043 & 3.78\% & 0.9719 & 0.2336 & 9.00\% & 0.00\% & 24.79\% \\
Tight Lo & Perfect & 0.0999 & 0.1336 & 3.75\% & 0.9998 & 0.0107 & 16.00\% & 1.00\% & 37.82\% \\
Normal & Perfect & 0.6999 & 0.7139 & 4.67\% & 0.9994 & 0.0560 & 17.00\% & 1.00\% & 35.29\% \\
Log-Uniform Low & Perfect & 0.1068 & 0.4944 & 43.42\% & 0.9990 & 0.0300 & 6.00\% & 1.00\% & 33.33\% \\
Log-Uniform High & Perfect & 0.9333 & 0.9608 & 41.21\% & 0.7147 & 8.5918 & 1.00\% & 0.00\% & 9.32\% \\
Bell & Perfect & 0.4987 & 0.5443 & 9.11\% & 0.9997 & 0.0352 & 17.00\% & 0.00\% & 39.16\% \\
\midrule
\multicolumn{10}{c}{\textit{underconfident}} \\
\midrule
Uniform & Underconf $0.2+0.8c$ & 0.6019 & 0.7353 & 33.54\% & 0.7913 & 0.1045 & 0.00\% & 0.00\% & 0.59\% \\
Skew High & Underconf $0.2+0.8c$ & 0.8850 & 0.9102 & 21.91\% & 0.7617 & 88.6478 & 0.00\% & 0.00\% & 5.71\% \\
Skew Low & Underconf $0.2+0.8c$ & 0.3122 & 0.4623 & 21.85\% & 0.8029 & 0.0157 & 0.00\% & 0.00\% & 0.00\% \\
Bimodal & Underconf $0.2+0.8c$ & 0.5994 & 0.8480 & 62.10\% & 0.7849 & 5.6867 & 0.00\% & 0.00\% & 3.39\% \\
Tight Hi & Underconf $0.2+0.8c$ & 0.9199 & 0.9229 & 3.73\% & 0.7929 & 0.2336 & 3.00\% & 0.00\% & 6.06\% \\
Tight Lo & Underconf $0.2+0.8c$ & 0.2824 & 0.3085 & 3.63\% & 0.7971 & 0.0107 & 0.00\% & 0.00\% & 0.00\% \\
Normal & Underconf $0.2+0.8c$ & 0.7580 & 0.7695 & 4.76\% & 0.8012 & 0.0560 & 0.00\% & 0.00\% & 0.00\% \\
Log-Uniform Low & Underconf $0.2+0.8c$ & 0.2824 & 0.5981 & 44.04\% & 0.8033 & 0.0300 & 0.00\% & 0.00\% & 0.00\% \\
Log-Uniform High & Underconf $0.2+0.8c$ & 0.9480 & 0.9694 & 41.21\% & 0.6383 & 8.5918 & 1.00\% & 0.00\% & 4.85\% \\
Bell & Underconf $0.2+0.8c$ & 0.5982 & 0.6350 & 9.15\% & 0.8016 & 0.0352 & 0.00\% & 0.00\% & 0.00\% \\
Uniform & Underconf $\sqrt{c}$ & 0.6670 & 0.8006 & 40.14\% & 0.6127 & 0.1045 & 0.00\% & 0.00\% & 0.77\% \\
Skew High & Underconf $\sqrt{c}$ & 0.9202 & 0.9396 & 24.30\% & 0.5230 & 88.6478 & 1.00\% & 1.00\% & 2.03\% \\
Skew Low & Underconf $\sqrt{c}$ & 0.3137 & 0.5482 & 34.20\% & 0.7790 & 0.0157 & 0.00\% & 0.00\% & 0.00\% \\
Bimodal & Underconf $\sqrt{c}$ & 0.6169 & 0.8834 & 69.59\% & 0.6306 & 5.6867 & 0.00\% & 0.00\% & 0.00\% \\
Tight Hi & Underconf $\sqrt{c}$ & 0.9487 & 0.9506 & 3.66\% & 0.4973 & 0.2336 & 0.00\% & 0.00\% & 0.65\% \\
Tight Lo & Underconf $\sqrt{c}$ & 0.2970 & 0.3562 & 8.44\% & 0.7769 & 0.0107 & 0.00\% & 0.00\% & 0.00\% \\
Normal & Underconf $\sqrt{c}$ & 0.8350 & 0.8433 & 5.01\% & 0.5439 & 0.0560 & 0.00\% & 0.00\% & 0.00\% \\
Log-Uniform Low & Underconf $\sqrt{c}$ & 0.2135 & 0.6645 & 57.37\% & 0.8515 & 0.0300 & 0.00\% & 0.00\% & 0.00\% \\
Log-Uniform High & Underconf $\sqrt{c}$ & 0.9615 & 0.9787 & 44.70\% & 0.3393 & 8.5918 & 0.00\% & 0.00\% & 1.91\% \\
Bell & Underconf $\sqrt{c}$ & 0.6949 & 0.7270 & 10.55\% & 0.6002 & 0.0352 & 0.00\% & 0.00\% & 0.00\% \\
Uniform & Random over c & 0.7503 & 0.8331 & 33.18\% & 0.5022 & 0.1045 & 0.00\% & 0.00\% & 0.00\% \\
Skew High & Random over c & 0.9278 & 0.9438 & 22.13\% & 0.4553 & 88.6478 & 0.00\% & 0.00\% & 1.01\% \\
Skew Low & Random over c & 0.5727 & 0.6693 & 22.63\% & 0.4975 & 0.0157 & 0.00\% & 0.00\% & 0.00\% \\
Bimodal & Random over c & 0.7491 & 0.9049 & 62.11\% & 0.4865 & 5.6867 & 0.00\% & 0.00\% & 0.64\% \\
Tight Hi & Random over c & 0.9493 & 0.9512 & 3.74\% & 0.5199 & 0.2336 & 1.00\% & 1.00\% & 1.55\% \\
Tight Lo & Random over c & 0.5496 & 0.5656 & 3.57\% & 0.5004 & 0.0107 & 0.00\% & 0.00\% & 0.00\% \\
Normal & Random over c & 0.8488 & 0.8559 & 4.69\% & 0.5042 & 0.0560 & 0.00\% & 0.00\% & 0.00\% \\
Log-Uniform Low & Random over c & 0.5523 & 0.7480 & 43.72\% & 0.5036 & 0.0300 & 0.00\% & 0.00\% & 0.00\% \\
Log-Uniform High & Random over c & 0.9665 & 0.9802 & 41.10\% & 0.5388 & 8.5918 & 1.00\% & 0.00\% & 2.80\% \\
Bell & Random over c & 0.7486 & 0.7719 & 9.29\% & 0.5004 & 0.0352 & 0.00\% & 0.00\% & 0.00\% \\
\bottomrule
\end{tabular}
\caption{Results for all distributions and calibration modes, averaged over 100 repetitions, N = 1000. Part A. }
\label{tab:report_all_indicators_A}
\end{table}

\begin{table}[ht]
\centering
\tiny
\begin{tabular}{llrrrrrrrr}
\toprule
\textbf{Distribution} & \textbf{Calibration} & $\overline{\textbf{Acc}}$ & $\overline{\textbf{cwA}}$ & $\overline{\textbf{gain}}$ & $\overline{\textbf{CSR}}$ & $\overline{\boldsymbol{\sigma}_{\mathrm{CSR}}}$ & $>1\boldsymbol{\sigma}_{\mathrm{CSR}}$ &  $>3\boldsymbol{\sigma}_{\mathrm{CSR}}$ & $\overline{\textbf{P}_{\textbf{risk}}}$ \\
\midrule
\multicolumn{10}{c}{\textit{overconfident}} \\
\midrule
Uniform & Overconf $1 - \sqrt{1 - c}$ & 0.3345 & 0.4677 & 20.03\% & 1.9643 & 0.1045 & 100.00\% & 96.00\% & 99.94\% \\
Skew High & Overconf $1 - \sqrt{1 - c}$ & 0.6860 & 0.7249 & 12.39\% & 10.4461 & 88.6478 & 74.00\% & 26.00\% & 89.92\% \\
Skew Low & Overconf $1 - \sqrt{1 - c}$ & 0.0804 & 0.1966 & 12.65\% & 1.1012 & 0.0157 & 100.00\% & 100.00\% & 100.00\% \\
Bimodal & Overconf $1 - \sqrt{1 - c}$ & 0.3849 & 0.6530 & 43.62\% & 5.7695 & 5.6867 & 68.00\% & 30.00\% & 87.42\% \\
Tight Hi & Overconf $1 - \sqrt{1 - c}$ & 0.7014 & 0.7080 & 2.24\% & 4.4671 & 0.2336 & 100.00\% & 100.00\% & 100.00\% \\
Tight Lo & Overconf $1 - \sqrt{1 - c}$ & 0.0517 & 0.0699 & 1.92\% & 1.0554 & 0.0107 & 100.00\% & 99.00\% & 100.00\% \\
Normal & Overconf $1 - \sqrt{1 - c}$ & 0.4622 & 0.4760 & 2.57\% & 1.9201 & 0.0560 & 100.00\% & 100.00\% & 100.00\% \\
Log-Uniform Low & Overconf $1 - \sqrt{1 - c}$ & 0.0655 & 0.3269 & 27.99\% & 1.1650 & 0.0300 & 99.00\% & 85.00\% & 99.49\% \\
Log-Uniform High & Overconf $1 - \sqrt{1 - c}$ & 0.8614 & 0.8964 & 25.28\% & 134.9216 & 8.5918 & 96.00\% & 79.00\% & 98.05\% \\
Bell & Overconf $1 - \sqrt{1 - c}$ & 0.2993 & 0.3322 & 4.70\% & 1.4764 & 0.0352 & 100.00\% & 100.00\% & 100.00\% \\
Uniform & Overconf $0.5c$ & 0.2510 & 0.3344 & 11.15\% & 8.8308 & 0.1045 & 100.00\% & 100.00\% & 100.00\% \\
Skew High & Overconf $0.5c$ & 0.4277 & 0.4436 & 2.78\% & 57168.2115 & 88.6478 & 100.00\% & 100.00\% & 100.00\% \\
Skew Low & Overconf $0.5c$ & 0.0713 & 0.1647 & 10.06\% & 1.1243 & 0.0157 & 100.00\% & 100.00\% & 100.00\% \\
Bimodal & Overconf $0.5c$ & 0.2511 & 0.4061 & 20.71\% & 65437.4015 & 5.6867 & 100.00\% & 100.00\% & 100.00\% \\
Tight Hi & Overconf $0.5c$ & 0.4516 & 0.4534 & 0.33\% & 28.5956 & 0.2336 & 100.00\% & 100.00\% & 100.00\% \\
Tight Lo & Overconf $0.5c$ & 0.0500 & 0.0665 & 1.74\% & 1.0576 & 0.0107 & 100.00\% & 100.00\% & 100.00\% \\
Normal & Overconf $0.5c$ & 0.3524 & 0.3594 & 1.08\% & 2.5909 & 0.0560 & 100.00\% & 100.00\% & 100.00\% \\
Log-Uniform Low & Overconf $0.5c$ & 0.0545 & 0.2534 & 21.05\% & 1.5712 & 0.0300 & 100.00\% & 99.00\% & 100.00\% \\
Log-Uniform High & Overconf $0.5c$ & 0.4683 & 0.4823 & 2.63\% & 37298.3427 & 8.5918 & 100.00\% & 100.00\% & 100.00\% \\
Bell & Overconf $0.5c$ & 0.2478 & 0.2700 & 2.96\% & 1.6257 & 0.0352 & 100.00\% & 100.00\% & 100.00\% \\
Uniform & Random under c & 0.2487 & 0.3321 & 11.10\% & 6.1840 & 0.1045 & 100.00\% & 100.00\% & 100.00\% \\
Skew High & Random under c & 0.4293 & 0.4450 & 2.75\% & 116891.8558 & 88.6478 & 99.00\% & 99.00\% & 99.53\% \\
Skew Low & Random under c & 0.0711 & 0.1671 & 10.35\% & 1.1225 & 0.0157 & 100.00\% & 100.00\% & 100.00\% \\
Bimodal & Random under c & 0.2505 & 0.4061 & 20.78\% & 23900.3248 & 5.6867 & 100.00\% & 98.00\% & 99.84\% \\
Tight Hi & Random under c & 0.4502 & 0.4521 & 0.34\% & 29.2531 & 0.2336 & 100.00\% & 100.00\% & 100.00\% \\
Tight Lo & Random under c & 0.0508 & 0.0687 & 1.89\% & 1.0565 & 0.0107 & 100.00\% & 100.00\% & 100.00\% \\
Normal & Random under c & 0.3500 & 0.3567 & 1.04\% & 2.5731 & 0.0560 & 100.00\% & 100.00\% & 100.00\% \\
Log-Uniform Low & Random under c & 0.0550 & 0.2537 & 21.04\% & 1.6221 & 0.0300 & 100.00\% & 100.00\% & 100.00\% \\
Log-Uniform High & Random under c & 0.4664 & 0.4798 & 2.52\% & 36470.1166 & 8.5918 & 100.00\% & 100.00\% & 100.00\% \\
Bell & Random under c & 0.2495 & 0.2718 & 2.97\% & 1.6220 & 0.0352 & 100.00\% & 100.00\% & 100.00\% \\
\bottomrule
\end{tabular}
\small
\caption{Results for all distributions and calibration modes, averaged over 100 repetitions, N = 1000. Part B. }
\label{tab:report_all_indicators_B}
\end{table}

\begin{table}[ht]
\centering
\footnotesize
\begin{tabular}{lrr}
\toprule
\textbf{Calibration} & \textbf{cwAUC $>$ AUC} & \textbf{cwAUC $<$ AUC} \\
\midrule
Random 0.5 & 0 & 10 \\
Perfect & 9 & 1 \\
Underconf $0.2+0.8c$ & 9 & 1 \\
Underconf $c^{0.5}$ & 9 & 1 \\
Random over c & 10 & 0 \\
Overconf $1- \sqrt{1-c}$ & 8 & 2 \\
Overconf $\sqrt{c}$ & 2 & 8 \\
Random under c & 5 & 5 \\
\bottomrule
\end{tabular}
\caption{Count of distributions where cwAUC (macro) exceeds or falls below AUC (macro), across all calibration modes. }
\label{tab:cwauc_summary}
\end{table}

\begin{table}[ht]
\centering
\footnotesize
\begin{tabular}{lrrrrrrrr}
\toprule
\textbf{Distribution} & \textbf{Acc} & \textbf{cwA} & $\boldsymbol{\Delta}$\textbf{cwA} & \textbf{gain} & \textbf{AUC} & \textbf{cwAUC} & $\boldsymbol{\Delta}$\textbf{AUC } & \textbf{AUC gain} \\
\midrule
Uniform & 0.5130 & 0.6949 & +0.1819 & 37.34\% & 0.6804 & 0.8142 & +0.1338 & 41.85\% \\
Skew High & 0.8510 & 0.8799 & +0.0289 & 19.43\% & 0.9324 & 0.9468 & +0.0143 & 21.18\% \\
Skew Low & 0.1410 & 0.3484 & +0.2074 & 24.15\% & 0.5025 & 0.5311 & +0.0286 & 5.75\% \\
Bimodal & 0.5010 & 0.8237 & +0.3227 & 64.67\% & 0.7201 & 0.9274 & +0.2073 & 74.06\% \\
Tight Hi & 0.8990 & 0.9023 & +0.0033 & 3.29\% & 0.9297 & 0.9326 & +0.0030 & 4.24\% \\
Tight Lo & 0.1000 & 0.1324 & +0.0324 & 3.60\% & 0.5000 & 0.4987 & -0.0013 & -0.27\% \\
Normal & 0.7240 & 0.7372 & +0.0132 & 4.78\% & 0.7776 & 0.7904 & +0.0128 & 5.75\% \\
Log-Uniform Low & 0.1090 & 0.5007 & +0.3917 & 43.96\% & 0.5081 & 0.6775 & +0.1693 & 34.43\% \\
Log-Uniform High & 0.9470 & 0.9692 & +0.0222 & 41.87\% & 0.9917 & 0.9940 & +0.0023 & 28.19\% \\
Bell & 0.4900 & 0.5371 & +0.0471 & 9.24\% & 0.5704 & 0.5996 & +0.0292 & 6.79\% \\
\bottomrule
\end{tabular}
\caption{Confidence-weighted accuracy and macro-averaged AUC for synthetic distributions under perfect calibration. }
\label{tab:diffs_synthetic}
\end{table}

\begin{table}[ht]
\centering
\footnotesize
\begin{tabular}{lrrrrrrrr}
\toprule
\textbf{Dataset} & \textbf{Acc} & \textbf{cwA} & $\boldsymbol{\Delta}$\textbf{cwA} & \textbf{gain} & \textbf{AUC} & \textbf{cwAUC} & $\boldsymbol{\Delta}$\textbf{AUC } & \textbf{AUC gain} \\
\midrule
\multicolumn{9}{c}{\textit{no post-hoc calibration}} \\
\midrule
accelerometer & 0.9755 & 0.9807 & +0.0052 & 21.13\% & 0.9364 & 0.9312 & -0.0052 & -7.57\% \\
adult\_income & 0.8408 & 0.8698 & +0.0290 & 18.21\% & 0.9251 & 0.9433 & +0.0183 & 24.37\% \\
bank\_marketing & 0.9052 & 0.9236 & +0.0184 & 19.43\% & 0.9474 & 0.9576 & +0.0102 & 19.40\% \\
breast\_cancer & 0.9737 & 0.9763 & +0.0026 & 9.79\% & 0.9897 & 0.9823 & -0.0074 & -41.92\% \\
cifar10 & 0.5858 & 0.6673 & +0.0815 & 19.69\% & 0.9195 & 0.9406 & +0.0211 & 26.17\% \\
covertype & 0.9587 & 0.9693 & +0.0106 & 25.61\% & 0.9983 & 0.9988 & +0.0005 & 29.34\% \\
credit & 0.7948 & 0.8126 & +0.0178 & 8.65\% & 0.7625 & 0.7729 & +0.0103 & 4.36\% \\
diabetes & 0.6818 & 0.7119 & +0.0301 & 9.46\% & 0.7804 & 0.7973 & +0.0169 & 7.72\% \\
higgs\_boson & 0.8352 & 0.8606 & +0.0253 & 15.38\% & 0.9101 & 0.9267 & +0.0166 & 18.49\% \\
liver & 0.7350 & 0.7600 & +0.0250 & 9.44\% & 0.7597 & 0.7787 & +0.0189 & 7.87\% \\
mnist & 0.9765 & 0.9835 & +0.0070 & 29.86\% & 0.9996 & 0.9997 & +0.0001 & 18.84\% \\
naticusdroid & 0.9724 & 0.9768 & +0.0044 & 15.84\% & 0.9940 & 0.9947 & +0.0007 & 12.16\% \\
obesity & 0.9811 & 0.9861 & +0.0050 & 26.57\% & 0.9988 & 0.9981 & -0.0007 & -36.05\% \\
phishing & 0.9729 & 0.9792 & +0.0064 & 23.46\% & 0.9960 & 0.9962 & +0.0002 & 3.93\% \\
wine & 0.6385 & 0.6778 & +0.0393 & 10.88\% & 0.8174 & 0.8342 & +0.0168 & 9.20\% \\
\midrule
\multicolumn{9}{c}{\textit{isotonic}} \\
\midrule
accelerometer & 0.9755 & 0.9854 & +0.0100 & 40.57\% & 0.9366 & 0.9220 & -0.0146 & -18.69\% \\
adult\_income & 0.8408 & 0.8928 & +0.0519 & 32.64\% & 0.9244 & 0.9465 & +0.0222 & 29.34\% \\
bank\_marketing & 0.9052 & 0.9373 & +0.0321 & 33.81\% & 0.9468 & 0.9594 & +0.0126 & 23.73\% \\
breast\_cancer & 0.9737 & 0.9732 & -0.0005 & -1.93\% & 0.9570 & 0.9626 & +0.0056 & 12.97\% \\
cifar10 & 0.5858 & 0.6724 & +0.0866 & 20.91\% & 0.9190 & 0.9417 & +0.0227 & 28.05\% \\
covertype & 0.9587 & 0.9684 & +0.0098 & 23.65\% & 0.9982 & 0.9986 & +0.0004 & 22.90\% \\
credit & 0.7948 & 0.8283 & +0.0335 & 16.32\% & 0.7604 & 0.7494 & -0.0110 & -4.40\% \\
diabetes & 0.6818 & 0.7388 & +0.0570 & 17.91\% & 0.7675 & 0.7935 & +0.0260 & 11.18\% \\
higgs\_boson & 0.8352 & 0.8660 & +0.0308 & 18.71\% & 0.9099 & 0.9263 & +0.0164 & 18.20\% \\
liver & 0.7350 & 0.7938 & +0.0588 & 22.19\% & 0.7270 & 0.7547 & +0.0278 & 10.16\% \\
mnist & 0.9765 & 0.9836 & +0.0071 & 30.33\% & 0.9994 & 0.9995 & +0.0001 & 14.71\% \\
naticusdroid & 0.9724 & 0.9771 & +0.0047 & 17.12\% & 0.9937 & 0.9946 & +0.0010 & 15.09\% \\
obesity & 0.9811 & 0.9858 & +0.0047 & 24.68\% & 0.9966 & 0.8455 & -0.1511 & -97.78\% \\
phishing & 0.9729 & 0.9805 & +0.0077 & 28.26\% & 0.9952 & 0.9960 & +0.0008 & 17.19\% \\
wine & 0.6385 & 0.6778 & +0.0393 & 10.88\% & 0.8174 & 0.8342 & +0.0168 & 9.20\% \\
\midrule
\multicolumn{9}{c}{\textit{platt}} \\
\midrule
accelerometer & 0.9755 & 0.9842 & +0.0087 & 35.51\% & 0.9364 & 0.9256 & -0.0107 & -14.45\% \\
adult\_income & 0.8408 & 0.8830 & +0.0421 & 26.48\% & 0.9251 & 0.9403 & +0.0152 & 20.27\% \\
bank\_marketing & 0.9052 & 0.9338 & +0.0286 & 30.15\% & 0.9474 & 0.9565 & +0.0091 & 17.23\% \\
breast\_cancer & 0.9737 & 0.9747 & +0.0010 & 3.76\% & 0.9897 & 0.9808 & -0.0089 & -46.57\% \\
cifar10 & 0.5858 & 0.6525 & +0.0667 & 16.10\% & 0.9106 & 0.9290 & +0.0184 & 20.61\% \\
covertype & 0.9587 & 0.9660 & +0.0073 & 17.77\% & 0.9978 & 0.9981 & +0.0003 & 13.33\% \\
credit & 0.7948 & 0.8260 & +0.0311 & 15.18\% & 0.7625 & 0.7481 & -0.0145 & -5.74\% \\
diabetes & 0.6818 & 0.7303 & +0.0485 & 15.23\% & 0.7804 & 0.7943 & +0.0140 & 6.36\% \\
higgs\_boson & 0.8352 & 0.8599 & +0.0247 & 14.98\% & 0.9101 & 0.9217 & +0.0116 & 12.92\% \\
liver & 0.7350 & 0.7802 & +0.0452 & 17.05\% & 0.7597 & 0.7475 & -0.0123 & -4.86\% \\
mnist & 0.9765 & 0.9792 & +0.0027 & 11.68\% & 0.9996 & 0.9996 & +0.0000 & 1.73\% \\
naticusdroid & 0.9724 & 0.9748 & +0.0024 & 8.68\% & 0.9940 & 0.9943 & +0.0003 & 4.98\% \\
obesity & 0.9811 & 0.9839 & +0.0028 & 14.79\% & 0.9987 & 0.9977 & -0.0010 & -43.25\% \\
phishing & 0.9729 & 0.9775 & +0.0047 & 17.16\% & 0.9960 & 0.9959 & -0.0002 & -4.15\% \\
wine & 0.6385 & 0.6778 & +0.0393 & 10.88\% & 0.8174 & 0.8342 & +0.0168 & 9.20\% \\
\bottomrule
\end{tabular}
\caption{Results for raw XGBoost scores and after Platt scaling post-hoc calibration across all datasets. AUC and cwAUC scores are macro-averaged. }
\label{tab:full_raw_isotonic_and_platt}
\end{table}

\begin{figure}
  \centering
  \includegraphics[width=0.8\linewidth]{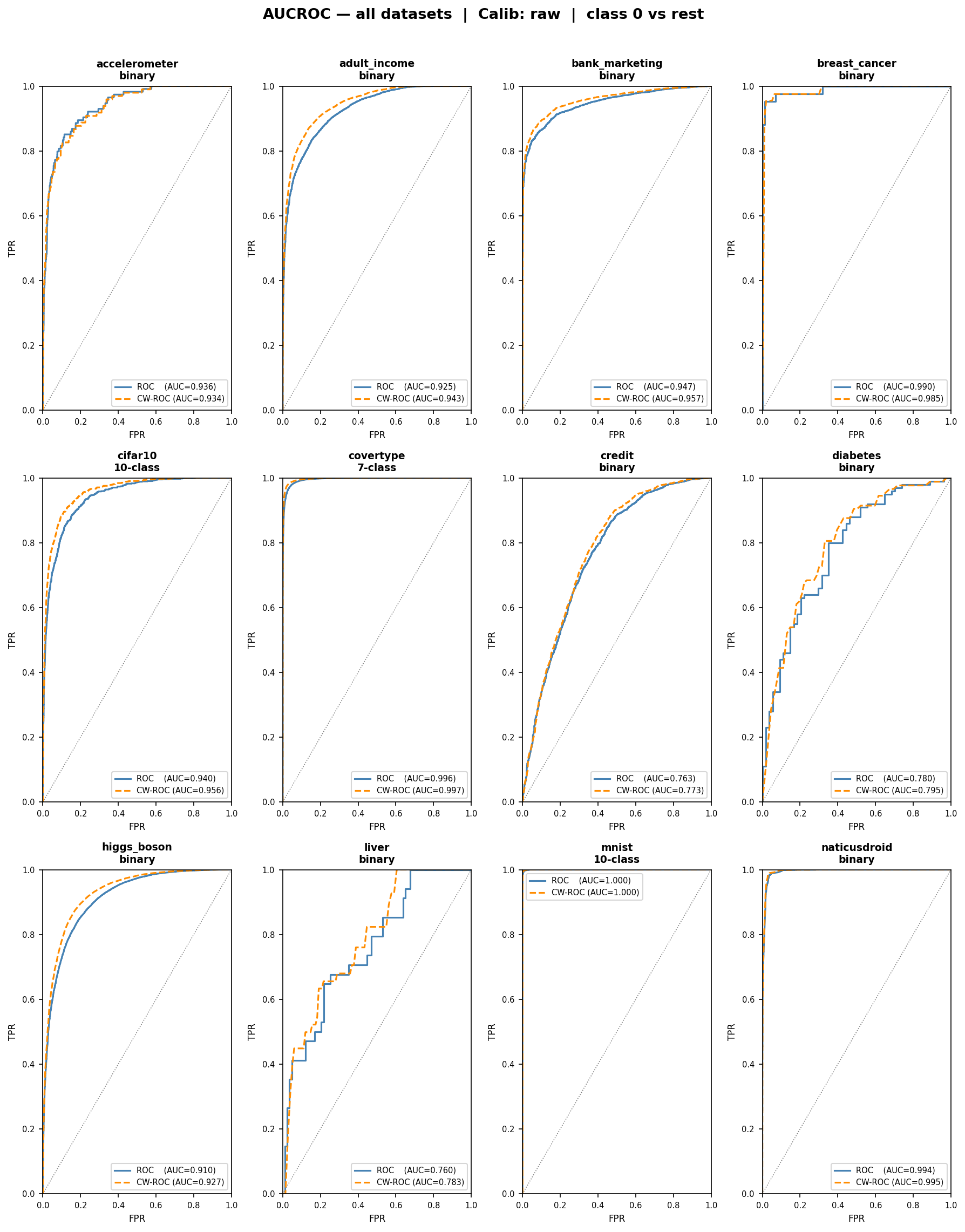}%
\caption{ROC and confidence-weighted ROC curves for class 0 of various datasets under Platt post-hoc calibration. }
\end{figure}


\end{document}